\pdfoutput=1 

\documentclass[letterpaper]{article} 
\usepackage{aaai18}  
\usepackage{times}  
\usepackage{helvet}  
\usepackage{courier}  
\usepackage{url}  
\usepackage{graphicx}  
\usepackage[ruled,linesnumbered]{algorithm2e}
\usepackage{booktabs}
\usepackage{amsthm}
\usepackage{amsfonts}
\usepackage{mathtools}
\usepackage{thmtools}
\usepackage{thm-restate}
\usepackage{paralist}
\usepackage{ifdraft}

\declaretheorem[name=Lemma]{lemma}

\declaretheorem[name=Definition]{definition}
\declaretheorem[name=Example]{example}

\newif\ifextendedversion
\extendedversiontrue 

\renewcommand\vec{\mathbf}

\newcommand{\myparagraph}[1]{\smallskip \noindent \textbf{#1}\ }

\newcommand{\segment}[2]{#1[#2]}
\newcommand{\head}{\ensuremath{\operatorname{head}}}
\newcommand{\body}{\ensuremath{\operatorname{body}}}

\newcommand{\program}{\ensuremath{\Pi}}
\newcommand{\radius}{\ensuremath{\operatorname{rad}}}
\newcommand{\offset}{\ensuremath{\Delta}}

\newcommand{\edb}[1]{\ensuremath{\mathit{#1}}}
\newcommand{\idb}[1]{\ensuremath{\mathit{#1}}}

\newcommand{\alignhere}{\ensuremath{& \;}}

\newcommand{\rank}{\ensuremath{\operatorname{rank}}}

\newcommand{\criticaldomain}[1]{\ensuremath{{\Omega_{#1}}}}
\newcommand{\criticalquery}[1]{\ensuremath{{\Theta_{#1}}}}
\newcommand{\criticalupdate}[1]{\ensuremath{{\Upsilon_{#1}}}}

\newcommand{\nrcriticalupdate}[1]{\ensuremath{{\Upsilon^\mathrm{b}_{#1}}}}

\newcommand{\hard}{\textnormal{-hard}}
\newcommand{\complete}{\textnormal{-complete}}
\newcommand{\pspace}{\textsc{PSpace}}

\newcommand{\pspacecomplete}{\pspace\complete}
\newcommand{\ptime}{\textsc{P}}

\newcommand{\logspace}{\textsc{LogSpace}}

\newcommand{\aczero}{\textsc{AC}\ensuremath{^0}}

\newcommand{\conexptime}{\textnormal{co-}\textsc{NExp}}
\newcommand{\conexptimehard}{\conexptime\hard}
\newcommand{\conexptimecomplete}{\conexptime\complete}

\newcommand{\connected}{connected}

\newcommand{\tin}{\tau_{\mathit{in}}}
\newcommand{\tout}{\tau_{\mathit{out}}}
\newcommand{\tmem}{\tau_{\mathit{mem}}}

\nocopyright{}

\begin{document}

\title{Stream Reasoning in Temporal Datalog}
\author{Alessandro Ronca, Mark Kaminski, Bernardo Cuenca Grau, Boris Motik \and Ian Horrocks\\
     Department of Computer Science, University of Oxford, UK\\
     $\{$alessandro.ronca, mark.kaminski, bernardo.cuenca.grau, boris.motik, ian.horrocks$\}$@cs.ox.ac.uk\\
}

\maketitle

\begin{abstract}
In recent years, there has been an increasing interest in extending traditional
stream processing engines with logical, rule-based, reasoning capabilities.
This poses significant theoretical and practical challenges since rules can
derive new information and propagate it both towards past and future time
points; as a result, streamed query answers can depend on data that has not yet
been received, as well as on data that arrived far in the past. Stream
reasoning algorithms, however, must be able to stream out query answers as soon
as possible, and can only keep a limited number of previous input facts in
memory. In this paper, we propose novel reasoning problems to deal with these
challenges, and study their computational properties on Datalog extended with a
temporal sort and the successor function---a core rule-based language for
stream reasoning applications.

\end{abstract}

\section{Introduction}

Query processing over data streams is a key aspect of Big
Data applications. For instance, algorithmic trading 
relies on real-time analysis of stock tickers and financial news items \cite{nuti2011algorithmic};
oil and gas companies continuously monitor and analyse data coming from their wellsites in order 
to detect equipment malfunction and predict maintenance needs \cite{cosad2009wellsite}; 
network providers perform real-time analysis of network flow data
to identify traffic anomalies and DoS attacks \cite{munz2007real}.

In stream processing, an input data stream is seen as an unbounded, append-only, relation  of timestamped tuples,
where timestamps are either added
by the external device that issued the tuple  or by the stream management
system receiving it  \cite{babu2001continuous,babcock2002models}.
The analysis of the input stream is performed using a standing query, the answers to which 
are also issued as a stream. 
Most applications of stream processing require near real-time 
analysis using limited resources, which poses significant challenges to stream management systems.
On the one hand, systems must be able to
compute query answers over the partial data received so far as if the
entire (infinite) stream had been available; furthermore, they must stream query answers
out with the minimum possible delay.
On the other hand, due to memory limitations, 
systems can only keep a limited \emph{history} of previously received
input facts in memory to perform computations. 
These challenges have been addressed by 
extending  traditional database query languages with window constructs, which 
declaratively specify the finite part of the input stream relevant to the
answers at the current time \cite{arasu2006cql}.

In recent years, there has been an increasing interest in
extending traditional stream management systems with logical, rule-based,
reasoning capabilities \cite{barbieri2010incremental,calbimonte2010enabling,anicic2011ep,le2011native,zaniolo2012streamlog,ozcep2014stream,beck2015lars,dao2015comparing}.  
Rules can be very useful in stream processing applications for capturing complex analysis tasks in a declarative way, 
as well as for representing background knowledge
about the application domain. 
\begin{example} \label{ex:running-example}
    Consider a number of wind turbines scattered throughout the North Sea.
    Each turbine is equipped with a sensor, which
    continuously records temperature levels of key devices within the turbine 
    and sends those readings to a data centre monitoring the
    functioning of the turbines. Temperature levels are streamed by sensors
    using a ternary predicate $\mathit{Temp}$, whose 
    arguments identify the 
    device, the temperature level, and the time of the reading. A monitoring task in the data centre
    is to track the activation of cooling measures in each turbine, 
    record temperature-induced malfunctions and shutdowns, and identify parts  at risk of future malfunction.
    This task is captured by the following set of rules:
    \begin{align}
        \mathit{Temp}(x,\mathit{high},t)&\to\mathit{Flag}(x,t)
        \label{eq:flag}\\
        \mathit{Flag}(x,t)\land\mathit{Flag}(x,t+1)&\to\mathit{Cool}(x,t+1)
        \label{eq:cool}\\
        \mathit{Cool}(x,t)\land\mathit{Flag}(x,t+1)&\to\mathit{Shdn}(x,t+1)
        \label{eq:shdn}\\
        \mathit{Shdn}(x,t)&\to\mathit{Malfunc}(x,t-2) \label{eq:malfunc}\\
        \mathit{Shdn}(x,t)\land\mathit{Near}(x,y)&\to\mathit{AtRisk}(y,t)
        \label{eq:risk}\\
        \mathit{AtRisk}(x,t)&\to\mathit{AtRisk}(x,t+1) \label{eq:rec}
    \end{align}
    Rule \eqref{eq:flag} `flags' a device whenever a high temperature reading is
    received. 
    Rule \eqref{eq:cool} says that two consecutive flags on
    a device trigger cooling measures.
    Rule \eqref{eq:shdn} says that
    an additional consecutive flag after activating cooling measures triggers a
    pre-emptive shutdown.
    By Rule \eqref{eq:malfunc}, a shutdown is due to a
    malfunction that occurred when the first flag leading to
    shutdown was detected.
    Finally, Rules \eqref{eq:risk} and \eqref{eq:rec}
    identify devices located near a shutdown device as being at risk and
    propagate risk recursively into the future.
\end{example}

The power and flexibility provided by rules poses additional challenges.
As seen in our  example, 
rules can derive information and propagate it both towards past
and future time points. As a result, query answers 
can depend  on data that has not yet been received 
(thus preventing the system from streaming out answers as soon as new input arrives), 
as well as on data
that arrived far in the past (thus forcing the system to 
keep in memory a potentially large input history).

Towards developing a solid foundation for rule-based stream reasoning,
we propose in Section \ref{sec:stream-problems}  
a suite of decision problems
that can be exploited by a
stream reasoning algorithm to deal with the aforementioned challenges.

\begin{itemize}
    \item The \emph{definitive time point} (DTP) problem 
        is to check whether query
        answers to be issued at a given time $\tout$
        will remain unaffected by any future input data
        given the current history; if so,  $\tout$
        is \emph{definitive} and answers at $\tout$ 
        can be safely output
        by the algorithm.
    \item 
        The \emph{forgetting} problem is to determine whether facts received at 
        a given previous time point and recorded in the current history can be `forgotten', in that they cannot
        affect future query answers. Forgetting allows the algorithm to maintain
        as small a history as possible.
    \item The \emph{delay} problem is to check, given a time gap $d$, whether time
        point $\tin - d$ is definitive for each time point $\tin$ at which new input
        facts are received and each history up to $\tin$.  Delay can thus be seen as
        a data-independent variant of DTP: the delay $d$ can be computed offline
        before receiving any data, and the algorithm can then safely output answers
        at $\tin-d$ as data at $\tin$ is being received.
    \item The \emph{window size} problem is a data-independent variant of
        forgetting.  The task is to determine, given a window size $s$, whether
        all history facts at
        time points up to $\tin-s$ can be forgotten for each time $\tin$ at
        which new input facts are received and each history up to $\tin$.  A stream
        reasoning algorithm can compute $s$ in an offline phase and then, in the online phase,
        immediately delete all history facts
        older than $s$ time points as new data arrives.
\end{itemize}

In Section \ref{sec:complexity}, we proceed to the study of the computational properties of the aforementioned
problems.
For this, we consider as query language \emph{temporal Datalog}---negation-free Datalog with
a special temporal sort to which the successor function (or, equivalently, addition by a constant)
is applicable \cite{chomicki1988temporal}. This is a core temporal rule-based language, which
captures other prominent temporal languages \cite{abadi1989temporal,baudinet1992temporal}
and forms the basis of more expressive formalisms for stream reasoning recently proposed in the
literature  \cite{zaniolo2012streamlog,beck2015lars}.

We show in Section \ref{sec:definitive} that DTP is $\pspace$-complete in data complexity and becomes
tractable for nonrecursive queries under very mild additional restrictions; 
thus, DTP is no harder than query evaluation \cite{chomicki1988temporal}.
In Section \ref{sec:forgetting}, we show that
forgetting is undecidable; however, quite surprisingly, the aforementioned restrictions
to nonrecursive queries allows
us to regain not only decidability, but also tractability in data complexity.
In Section \ref{sec:delay}, we turn our attention to data-independent problems.
We show that both delay and window size are undecidable in general and become
$\conexptime$-complete for nonrecursive queries.

Our results show that, although stream reasoning problems are either intractable in data
complexity or undecidable in general, they become feasible in practice for
nonrecursive queries under very mild additional restrictions. On the one hand, the data-dependent problems (DTP and forgetting) become tractable
in data complexity (a very important requirement for achieving near real-time computation in practice);
on the other hand, although the data-independent problems (delay and window size) remain intractable, these are
one-time problems which 
only need to be solved once prior to receiving any input data.

\ifextendedversion
    The proofs of all results are given in the appendix of this
    paper.
\else
    The proofs of all results are given in an extended version
    of this paper \cite{extendedversion}.
\fi

\section{Preliminaries}
\label{sec:preliminaries}

\noindent \textbf{Syntax}\ 
A vocabulary consists of \emph{predicates}, \emph{constants} and
\emph{variables}, where constants are partitioned into \emph{objects} and 
integer \emph{time points} and 
variables are partitioned into \emph{object variables} and \emph{time
variables}.
An \emph{object term} is an object or an object variable.
A \emph{time term} is either a time point, a time variable, 
or an expression of the form $t + k$ where $t$ is a 
time variable, $k$ is an integer number,
and $+$ is the standard 
\emph{integer addition function}. 
The offset $\offset(s)$ of a time term $s$ equals zero if $s$ is a time variable
or a time point and it equals $k$ if $s$ is of the form $s= t+k$.

Predicates are partitioned into 
\emph{extensional} (EDB) and 
\emph{intensional} (IDB) and they come with a nonnegative integer
\emph{arity} $n$, where each position ${1 \leq i \leq n}$ is of either
\emph{object} or \emph{time sort}. 
A predicate is \emph{rigid} if
all its positions are of object sort and it is \emph{temporal} 
if the last position is 
of time sort and all other positions are of object sort.
An \emph{atom} is an expression $P(t_1, \ldots, t_n)$ where
$P$ is a predicate and each $t_i$ is a term of the required sort.
A \emph{rigid} atom (respectively, temporal, IDB,
EDB) is an atom involving a rigid predicate
(respectively, temporal, IDB, EDB).

A  \emph{rule} $r$ is of the form $\bigwedge_i \alpha_i
\rightarrow \alpha$, where $\alpha$ and each $\alpha_i$ are rigid or temporal
atoms, and $\alpha$ is IDB whenever $\bigwedge_i \alpha_i$ is non-empty.
Atom $\head(r) = \alpha$ is the \emph{head} of $r$, and $\body(r) = \bigwedge_i \alpha_i$
is the \emph{body} of $r$.
Rules are assumed to be \emph{safe}: each head variable
must occur in the body.
An \emph{instance} $r'$ of  $r$ is obtained by
applying a substitution to $r$.
A  \emph{program} $\program$ is a finite set of rules.
Predicate $P$ is
$\program$\emph{-dependent} on a predicate $P'$ if there is a rule of $\program$
with $P$ in the head and $P'$ in the body.
The rank $\rank(P, \Pi)$ of $P$ w.r.t.\ $\Pi$ is $0$ if
$P$ does not occur in head position in $\Pi$, and is the
maximum of the values $\rank(P') + 1$ for $P'$ a predicate such that $P$ is
$\Pi$-dependent on $P'$ otherwise. We write $\rank(P)$ for $\rank(P, \Pi)$ if
$\Pi$ is clear from the context. The rank $\rank(\Pi)$ of $\Pi$ is
the maximum rank of a predicate in $\Pi$.

A  \emph{query} is a pair $Q = \langle P_Q, \program_Q \rangle$ where $\program_Q$ is a program and
$P_Q$ is an IDB predicate in $\Pi_Q$; query $Q$ is \emph{temporal}
(\emph{rigid}) if $P_Q$ is a temporal (rigid) predicate. 
A term, atom, rule, or program is \emph{ground} if it contains no variables.
A \emph{fact} $\alpha$ is a ground, function-free rigid or temporal atom;
every fact $\alpha$ corresponds to a rule of the
form $\top \rightarrow \alpha$ where $\top$ denotes the empty conjunction, 
so we use $\alpha$ and $\top \rightarrow \alpha$ interchangeably.
A \emph{dataset} $D$ is a program consisting of EDB facts.
The \emph{$\tau$-segment} $\segment{D}{\tau}$ of dataset $D$
is the subset of $D$ containing all
rigid facts and all temporal facts with time argument
$\tau' > \tau$. 

A program $\Pi$ (respectively, query $Q$) is:
\emph{Datalog} if no temporal predicate occurs in $\Pi$ (in $\Pi_Q$);
and \emph{nonrecursive} if the directed graph induced by the $\Pi$-dependencies ($\Pi_Q$-dependencies) 
is acyclic.

\myparagraph{Semantics and standard reasoning}
Rules are interpreted in the standard way as universally quantified
first-order sentences. A Herbrand interpretation $\mathcal{H}$ is a
(possibly infinite) set of facts.
Interpretation $\mathcal{H}$ \emph{satisfies} a 
rigid atom $\alpha$ if $\alpha \in \mathcal{H}$, and it satisfies
a temporal atom $\alpha$ if evaluating the addition 
function in $\alpha$ yields a fact in $\mathcal{H}$.
The notion of satisfaction is extended to conjunctions of ground atoms,
rules and programs in the standard way.
If $\mathcal{H} \models \program$, then $\mathcal{H}$ is a \emph{model} of $\program$.
Program $\program$ \emph{entails} a fact $\alpha$, 
written $\program \models \alpha$, 
if $\mathcal{H} \models \program$ implies $\mathcal{H} \models \alpha$. 
The \emph{answers} to a query $Q$ over a dataset $D$, written $Q(D)$, 
are the tuples $\vec{a}$ of constants such that $\program_Q \cup D \models P_Q(\vec{a})$.
If $Q$ is a temporal query, we denote with $Q(D,\tau)$ the subset of answers in
$Q(D)$ referring to time point $\tau$.
Given an input query $Q$, dataset $D$ and tuple $\vec{a}$,
the \emph{query evaluation problem} is to check whether $\vec{a}$ is an
answer to $Q$ over $D$; the \emph{data complexity} of query evaluation is the
complexity when $Q$ is considered fixed.
Finally, a query $Q_1$ is \emph{contained} in a query $Q_2$, written $Q_1
\sqsubseteq Q_2$, if $Q_1(D) \subseteq Q_2(D)$ for every dataset $D$.
Given input queries $Q_1$ and $Q_2$,
the \emph{query containment problem} is to check
whether $Q_1$ is contained in $Q_2$.

\myparagraph{Complexity}
Query evaluation is
$\pspacecomplete$ in data complexity assuming that numbers are coded in unary 
\cite{chomicki1988temporal}.
Data complexity drops to the circuit class $\aczero$  for 
nonrecursive programs.
By standard results in nontemporal Datalog, containment of 
temporal queries is undecidable \cite{shmueli1993equivalence}.
Furthermore, it is $\conexptime$-hard for nonrecursive queries \cite{benedikt2010impact}.

\section{Stream Reasoning Problems}
\label{sec:stream-problems}

A stream reasoning algorithm  receives as input  a
query $Q$ and a stream of temporal EDB facts, and
produces as output a stream of answers to $Q$.
Both input facts and query answers are 
processed 
by increasing value of their
timestamps, where $\tin$ and $\tout$ represent the current times at which
input facts are received and query answers are being streamed out, respectively. 
Answers at  $\tout$ are only output
when the algorithm can determine that 
they cannot be affected by future input facts.
In turn, input facts received so far 
are  kept in a \emph{history dataset} $D$ since
future query answers can be influenced by facts received at an earlier time;
practical systems, however, have limited memory and hence 
the algorithm  must also \emph{forget} facts in the history as
soon as it can determine that they will not influence
future query answers.

\begin{algorithm}[t]
    \begin{footnotesize}
        \DontPrintSemicolon
        \SetKwRepeat{Loop}{loop}{end}
        \SetKwInput{Params}{Parameters}
        \Params{Temporal query $Q$}
        $D:=\emptyset,~\tin:=0,~\tout:=0,~\tmem:=0$\;
        \Loop{}{
            receive facts $U$ that hold at $\tin$ and set $D:=D\cup U$ \;
            \While{$\tout\le\tin$ $\textup{{\bfseries and}}$  $\text{$\tout$ is
                definitive for $Q,D,\tin$}$}{
                stream output $Q(D,\tout)$\;
                $\tout:=\tout+1$\;
            }
            \While{$\tmem<\tout$ $\textup{{\bfseries and}}$ $\text{$\tmem$ is forgettable for $Q,D,\tin,\tout$}$}{
                forget all temporal facts in $D$ holding at $\tmem$\;
                $\tmem:=\tmem+1$\;
            }
        $\tin:=\tin+1$\; }
        \caption{`Online' Stream Reasoning Algorithm}
    \end{footnotesize}
\end{algorithm}

Algorithms 1 and 2 provide two different realisations of such a stream reasoning
algorithm, which we refer to as \emph{online} and \emph{offline}, respectively.


The \emph{online algorithm} (see Algorithm 1) decides which answers to stream and
which history facts to forget `on the fly' as new input data arrives. The
algorithm records the latest time point $\tout$ for which answers have not yet
been streamed; 
as  $\tin$ increases and new data arrives, the algorithm checks 
(lines 4-7) whether answers at $\tout$ can now be streamed and, if so, it
continues incrementing $\tout$ until it finds a time point for which answers cannot be provided yet. 
This process relies on deciding whether
the considered $\tout$ are \emph{definitive}---that is, 
the answers to $Q$ at  $\tout$ for the history $D$ will remain stable even if $D$ were extended with an unknown
(and thus  arbitrary) set $U$
of future input facts.

\begin{definition}
    A \emph{$\tin$-history} $D$ is a dataset  consisting of 
    rigid facts and temporal facts with time argument at most $\tin$.
    A \emph{$\tin$-update} $U$ is a dataset  consisting of 
    temporal facts with time argument strictly greater than $\tin$.
\end{definition}

\begin{definition}
    \label{definition:delay}
    An instance $I$ of the \emph{Definitive Time Point} ($\textsc{DTP}$)
    problem is a tuple  $\langle Q,D,\tin,\tout\rangle$, with $Q$ a temporal query, $D$ a $\tin$-history
    and $\tout \leq \tin$. $\textsc{DTP}$ holds
    for $I$ iff $Q(D, \tout) = Q(D \cup U, \tout)$ for each $\tin$-update
    $U$.
\end{definition} 

\begin{example}
    Consider Example~\ref{ex:running-example}, and suppose we are interested in
    determining the time points at which a turbine malfunctions.  Thus, let the
    query $Q$ have output predicate $\mathit{Malfunc}$ and include rules
    \eqref{eq:flag}--\eqref{eq:malfunc} together with rule
    \begin{align*}
        \mathit{Temp}(x, \mathit{na}, t) &\to \mathit{Malfunc}(x,t)
    \end{align*}
    which defines an invalid reading as a malfunction. For a history $D$
    consisting of the fact $\mathit{Temp}(\mathit{a}, \mathit{high}, 0)$, we have
    that $\textsc{DTP}(Q, D, 0, 0)$ is false, since $Q(D, 0)$ is empty and
    $Q(D \cup U, 0)$ is not if the update $U$ contains
    $\mathit{Temp}(\mathit{a}, \mathit{high}, 1)$ and
    $\mathit{Temp}(\mathit{a}, \mathit{high}, 2)$.
    For a history $D'$ consisting of the fact
    $\mathit{Temp}(\mathit{a}, \mathit{na}, 0)$, we have that
    $\textsc{DTP}(Q, D', 0, 0)$ is true, since $Q(D', 0)$ already includes the
    only possible answer.
\end{example}

Algorithm 1 also records the latest time point $\tmem$ for which history facts
have not yet been forgotten. As $\tin$ increases, the algorithm checks in lines
8-11 whether all history facts at time $\tmem$ can now be forgotten and, if so,
it continues incrementing $\tmem$ until it finds a point where this is no longer
possible. For this, the algorithm decides whether the relevant $\tmem$
are \emph{forgettable}, in the sense that no future answer to $Q$ can be
affected by the history facts at $\tmem$.

\begin{definition}   
    \label{definition:forgetting}
    An instance $I$ of\/ $\textsc{Forget}$ is a tuple of the form
    $\langle Q,D,\tin,\tout,\tmem\rangle$, with $Q$ a temporal query, $D$ a $\tin$-history, and $\tmem \leq \tout \leq \tin$.
    $\textsc{Forget}$ holds for $I$ iff $Q(D \cup U, \tau) = Q(\segment{D}{\tmem} \cup U, \tau)$ for each
    $\tin$-update $U$ and each time point $\tau\ge\tout$.  
\end{definition}

\begin{example}
    Consider the query $Q$ with output predicate $\mathit{Shdn}$ and rules
    \eqref{eq:flag}--\eqref{eq:shdn} from Example~\ref{ex:running-example}.
    For a history $D$ consisting of facts
    $\mathit{Temp}(\mathit{a}, \mathit{high}, 0)$ and
    $\mathit{Temp}(\mathit{a}, \mathit{low}, 1)$, we have that
    $\textsc{Forget}(Q, D, 1,1,1)$ is true, since $Q(D[1] \cup U, 1)$ is empty for
    every $1$-update $U$.
    For a history $D'$ containing facts
    $\mathit{Temp}(\mathit{a}, \mathit{high}, 0)$ and
    $\mathit{Temp}(\mathit{a}, \mathit{high}, 1)$, we have that
    $\textsc{Forget}(Q, D, 1,1,1)$ is false, since $Q(D[1] \cup U, 1)$ is empty but
    $Q(D \cup U, 1)$ is not for the $1$-update containing
    $\mathit{Temp}(\mathit{a}, \mathit{high}, 2)$.
\end{example}


\begin{algorithm}[t]
    \begin{footnotesize}
        \DontPrintSemicolon
        \SetKwRepeat{Loop}{loop}{end}
        \SetKwInput{Params}{Parameters}
        \Params{Temporal query $Q$}  
        \caption{`Offline' Stream Reasoning Algorithm}
        $D:=\emptyset,~\tin:=0$\;
        compute minimal delay $d$ and minimal window size $s$ for $Q$\;
        \Loop{}{
            receive facts $U$ that hold at $\tin$ and set $D:=D\cup U$ \;
            \lIf{$\tin-d\ge 0$}{stream output $Q(D,\tin-d)$}
            forget all temporal facts in $D$ holding at $\tin-s$\;
        $\tin:=\tin+1$\; }
    \end{footnotesize}
\end{algorithm}

The \emph{offline algorithm} (Algorithm 2) precomputes 
the minimum  \emph{delay} $d$ and  \emph{window
size} $s$ for the standing query $Q$ in a way that is independent
from the input data stream. 

Intuitively, $d$ represents the
smallest time gap needed to ensure that, for any input stream and any time point $\tin$, 
the time point $\tout = \tin -d$ is definitive; in other words, that it is
always safe to stream answers with a delay $d$ relative to the currently processed input
facts. 

\begin{definition}
    \label{definition:di-delay}
    An instance $I$ of\/ $\textsc{Delay}$ is a pair 
    $\langle Q,d \rangle$, with $Q$ a temporal query and $d$ a nonnegative
    integer.  $\textsc{Delay}$ holds for $I$ iff $Q(D, \tin-d) = Q(D \cup U, \tin-d)$ 
    for each time
    point $\tin$, each $\tin$-history $D$, and each $\tin$-update $U$.
\end{definition}

\begin{example}
    In Example~\ref{ex:running-example},
    $0$ is a valid delay for the $\mathit{AtRisk}$ and $\mathit{Shdn}$ queries,
    and so is $2$ for the $\mathit{Malfunc}$ query.
\end{example}

In turn, $s$ represents the size of the smallest time interval for
which the history needs to be kept; in other words, for any input stream and any
time point $\tin$, it is safe to forget all history facts with timestamp smaller
than $\tin - s$.  

\begin{definition}
    \label{definition:di-forgetting}
    An instance $I$ of\/ $\textsc{Window}$ is a triple
    $\langle Q,d,s \rangle$, with $Q$ a temporal query and $d$ and $s$ nonnegative
    integers.  $\textsc{Window}$ holds for $I$ iff $Q(D \cup U, \tout) = Q(\segment{D}{\tin - s} \cup U, \tout)$
    for all time
    points $\tin$ and $\tout$ with $\tout > \tin-d$, each
    $\tin$-history
    $D$, and each $\tin$-update
    $U$.
\end{definition}

\begin{example}
    Consider Example~\ref{ex:running-example}.
    Assuming that we want to evaluate queries with delay $0$,
    a valid window size for the $\mathit{Shdn}$ query is $2$;
    whereas the $\mathit{AtRisk}$ query has no valid window size 
    (or, equivalently, the query requires a window of infinite size),
    since answers for that query can depend on facts arbitrarily far in the past.
\end{example}

Once the delay $d$ and window size $s$ have been determined, they remain fixed
during execution of the algorithm: indeed, as $\tin$ increases and new data
arrives in each iteration of the main loop, Algorithm 2 simply streams query
answers at $\tin-d$ and forgets all history facts at $\tin-s$. This is in
contrast to the online approach, where the algorithm had to decide in each
iteration of the main loop which answers to stream and which facts to forget.

\section{Complexity of Stream Reasoning}\label{sec:complexity}

We now start our investigation of the computational properties of 
the stream reasoning problems introduced in Section~\ref{sec:stream-problems}.
For all problems, we consider both the general case applicable to arbitrary
inputs and the restricted setting where the input queries are nonrecursive.
All our results assume that numbers are coded in unary.

\subsection{Definitive Time Point}\label{sec:definitive}
\label{subsec:dtp}

Let $I = \langle Q, D, \tin, \tout
\rangle$ be a fixed, but arbitrary, instance  of $\textsc{DTP}$ and denote with
$\criticaldomain{I}$  the set consisting of all objects in $\Pi_Q \cup D$ and a fresh object $o_I$ unique to $I$.

As stated in Definition \ref{definition:delay},  $\textsc{DTP}$ holds
for $I$ if and only if the query answers at time $\tout$ 
over the
history $D$ coincide with the answers at the same time point but over
$D$ extended with an arbitrary $\tin$-update
$U$. Note that, in addition to new facts over existing objects in $D$ and $Q$,
the update $U$ may also include facts about new objects. The following
proposition shows that, to decide \textsc{DTP}, it suffices to consider updates
involving only objects from $\criticaldomain{I}$. Intuitively, updates
containing fresh objects can be homomorphically embedded into updates over
$\criticaldomain{I}$ by mapping all fresh objects to $o_I$.

\begin{restatable}{proposition}{dtpcriticaldomain}
    \label{prop:delay}
    $\textsc{DTP}$ holds for $I$ iff $\vec{o} \in Q(D \cup U, \tout)$ implies
    $\vec{o} \in Q(D, \tout)$ for every tuple $\vec{o}$ over\/
    $\criticaldomain{I}$ and every $\tin$-update $U$ involving only objects in\/
    $\criticaldomain{I}$.
\end{restatable}

\myparagraph{The general case.} We next show that DTP is decidable and provide
tight complexity bounds. Our upper bounds are obtained by showing that, to
decide DTP, it suffices to consider a single \emph{critical update} and a
slight modification of the query, which we refer to as the \emph{critical
query}. Intuitively, the critical update is a dataset that contains all
possible facts at the next time point $\tin + 1$ involving EDB predicates from
$Q$ and objects in $\criticaldomain{I}$. In turn, the critical query extends
$Q$ with rules that propagate all facts in the critical update recursively into
the future.  The intention is that the answers to the critical query over $D$
extended with the critical update will capture the answers to $Q$ over $D$
extended with any arbitrary future update.
In the following definition, we use $\psi$ to denote the renaming mapping each
temporal EDB predicate to a fresh temporal IDB predicate of the same arity.

\begin{definition}\label{def:critical-query}
    Let $\edb{A}$ be a fresh unary temporal EDB predicate.  The \emph{critical
    update} $\criticalupdate{I}$ for $I$ is the $\tin$-update containing the
    fact $\edb{A}(\tin + 1)$, and all facts $P(\vec{o}, \tin + 1)$ for each
    temporal EDB predicate $P$ in $\Pi_Q$ and each tuple $\vec{o}$ over
    $\criticaldomain{I}$.

    Let $\idb{V}$ be a fresh unary temporal IDB predicate. 
    The \emph{critical query} $\criticalquery{I}$ for $I$ is the query where
    $P_{\criticalquery{I}} = P_Q$ and $\Pi_{\criticalquery{I}}$ is obtained from
    $\psi(\Pi_Q)$
    by 
    adding rule $ \edb{A}(t) \rightarrow \idb{V}(t)$, rule
    $ \idb{V}(t) \rightarrow V(t+1)$, and the following rules for each temporal
    EDB predicate $P$ occurring in $\Pi_Q$, where $P'=\psi(P)$:
    \begin{align*}
        P(\vec{x}, t) \rightarrow \alignhere P'(\vec{x}, t) \\
        \idb{V}(t+1) \wedge  P'(\vec{x}, t) \rightarrow \alignhere P'(\vec{x},
        t+1)
    \end{align*}
\end{definition}

The construction of the critical query and update ensures, on
the one hand, that $Q(D,\tout) = \criticalquery{I}(D,\tout)$ and,
on the other hand, that 
$Q(D\cup U,\tout) \subseteq \criticalquery{I}(D\cup\criticalupdate{I},\tout)$  for each
$\tin$-update $U$ involving only objects in $\criticaldomain{I}$. 
We can exploit these properties, together with Proposition \ref{prop:delay},  to
show that DTP can be decided  by checking whether, at $\tout$,
the answers to the critical query over $D$ remain the same if $D$ is extended
with the critical update.
\begin{restatable}{lemma}{dtpupper} 
    \label{lem:dtp-upper}
    $\textsc{DTP}$ holds for $I$ 
    iff   
    $\vec{o} \in \criticalquery{I}(D \cup \criticalupdate{I}, \tout)$ implies 
    $\vec{o} \in \criticalquery{I}(D, \tout)$ for every tuple $\vec{o}$ over 
    $\criticaldomain{I}$.
\end{restatable}
It follows from Lemma \ref{lem:dtp-upper} that, to decide DTP, we need to
perform two temporal query evaluation tests for each candidate tuple $\vec{o}$.
Since temporal query evaluation is feasible in $\pspace$ in data
complexity, then so is DTP because 
the number of candidate tuples $\vec{o}$ 
is polynomial if $Q$ is fixed.

Furthermore, query evaluation is reducible to $\textsc{DTP}$, and hence the
aforementioned $\pspace$ upper bound in data complexity is tight.

\begin{restatable}{theorem}{dtptheorem}
    \label{theorem:delayunrestricted} %
    $\textsc{DTP}$ is \pspacecomplete{} in data complexity.
\end{restatable}

\myparagraph{Nonrecursive queries}
We next show that $\textsc{DTP}$ becomes tractable in data complexity 
for nonrecursive queries. In the remainder of this section, we fix an
arbitrary instance $I = \langle Q, D, \tin, \tout \rangle$ of $\textsc{DTP}$, where 
$Q$ is nonrecursive.
We assume w.l.o.g.\ that $Q$ does not contain rigid atoms: each such atom 
$P(\vec{c})$ can be replaced with a temporal atom of the
form, e.g., $P'(\vec{c}, 0)$. We make the additional technical
assumption that each rule in $Q$ is restricted as follows.
\begin{definition}
    A rule is \emph{\connected{}} if it contains at most one temporal variable,
    which occurs in the head whenever it occurs in the body. A query is
    \connected{} if so are its rules.
\end{definition}
Restricting our arguments to \connected{} queries allows us to considerably simplify
definitions and proofs.

We start with the observation that the critical query of $I$ 
always includes recursive  rules that propagate information 
arbitrarily far
into the future (see Definition \ref{def:critical-query}).  As a result,
our general algorithm for DTP does not immediately provide an improved upper bound
in the nonrecursive case. 

The need for such recursive rules, however, is motivated by the fact that
the answers to a recursive query $Q$ at time $\tout$ may depend on facts at time
points $\tau$ arbitrarily far from $\tout$; in other
words, there is no bound $b$ for $I$ such that $|\tau-\tout| \leq b$
in every derivation of query answers at $\tout$
involving input facts at $\tau$. If $Q$ is nonrecursive, however, such a bound 
is guaranteed to exist and can be established based on the following notion of
program \emph{radius}.
Intuitively,
query answers at $\tout \leq \tin$ can only be influenced by future facts 
whose timestamp is located within the interval  $[\tin, \tout +  \radius(\Pi_Q)]$.

\begin{definition} \label{def:radius} %
    Let $r$ be a \connected{} rule mentioning a time variable. The \emph{radius}
    $\radius(r)$ of $r$ is the maximum of the values
    $\vert \offset(s) - \offset(s') \vert$ for $s$ the time argument in the head
    of $r$ and $s'$ the time argument in a body atom of $r$.  The \emph{radius}
    $\radius(\Pi)$ of a \connected{} program $\Pi$ is given by the number of
    rules in $\Pi$ multiplied by the maximum radius of a rule in $\Pi$.
\end{definition}

Thus, to show tractability of DTP, we 
identify a polynomially
bounded number of \emph{critical time points} using the radius of 
$\Pi_Q$
and argue that we can dispense with
the aforementioned recursive rules by
constructing a critical update  for
$I$ that includes all facts over these time points.
Since $\Pi_Q$ may contain explicit time points in rules, these also need to be 
taken into account when defining the relevant critical time points and the
corresponding critical update.

\begin{definition} \label{def:nr-critical-update} %
    Let $\tau_0$ be the maximum value between $\tout$ and the largest time point
    occurring in $\Pi_Q$.

    A time point $\tau$ is \emph{critical} $I$ if 
    $\tin < \tau \leq \tau_0 + \radius(\Pi_Q)$.  The \emph{bounded critical
    update} $\nrcriticalupdate{I}$ of $I$ consists of each fact
    $P(\vec{o}, \tau)$ with $P$ a temporal EDB predicate in $\Pi_Q$, $\vec{o}$ a
    tuple over $\criticaldomain{I}$, and $\tau$ a critical time point.
\end{definition}

The following lemma justifies the key property of the critical update
$\nrcriticalupdate{I}$, namely that the answers to $Q$ over
$D \cup \nrcriticalupdate{I}$ capture those over $D$ extended with any future
update.

\begin{restatable}{lemma}{nrdtpcriticalupdate}
    \label{lem:nr-basic}
    Let $a$ be the maximum radius of a rule in $\Pi_Q$ and let $\mathbf{T}$
    consist of $\tout$ and the time points in $\Pi_Q$. 
    If $\Pi_Q \cup D \cup U \models P(\vec{o}, \tau)$ for a $\tin$-update $U$
    involving only objects in $\criticaldomain{I}$ and a predicate $P$ in
    $\Pi_Q$, and
    $|\tau - \tau'|\le a \cdot (\rank(\Pi_Q)-\rank(P))$ for some
    $\tau'\in\mathbf{T}$, then
    $\Pi_Q \cup D \cup \nrcriticalupdate{I} \models P(\vec{o}, \tau)$.
\end{restatable}

We are now ready to establish the analogue to Lemma \ref{lem:dtp-upper}
in the nonrecursive case.

\begin{restatable}{lemma}{nrdtpmain}
    \label{lem:nonrec-correct}
    $\textsc{DTP}$ holds for $I$ iff
    $\vec{o} \in Q(D \cup \nrcriticalupdate{I}, \tout)$ implies
    $\vec{o} \in Q(D, \tout)$ for every tuple $\vec{o}$ over $\criticaldomain{I}$.
\end{restatable}

Tractability of $\textsc{DTP}$ then follows from
Lemma \ref{lem:nonrec-correct} and the tractability of query evaluation for
nonrecursive programs. 

\begin{restatable}{theorem}{nrdtptheorem}
    \textsc{DTP} is in \ptime{} in data complexity if restricted to nonrecursive
    \connected{} queries.
\end{restatable}

\subsection{Forgetting}\label{sec:forgetting}
\label{subsec:hardness-forgetting}

We now move on to the forgetting problem as 
given in Definition \ref{definition:forgetting}. Unfortunately, in contrast to $\textsc{DTP}$, forgetting 
is undecidable. 
This follows by a reduction from containment of nontemporal Datalog queries---a
well-known undecidable problem \cite{shmueli1993equivalence}.

\begin{restatable}{theorem}{forgettingtheorem}
    \label{theorem:forgetting}
    $\textsc{Forget}$ is undecidable.
\end{restatable}

In the remainder of this section we show that, by restricting ourselves to
nonrecursive input queries, we can regain not only decidability of forgetting,
but also tractability in data complexity.  Let
$I = \langle Q, D, \tin, \tout, \tmem \rangle$ be an arbitrary instance of
$\textsc{Forget}$ where $Q$ is nonrecursive.  We adopt the same technical
assumptions as in Section \ref{sec:definitive} for $\textsc{DTP}$ in the
nonrecursive case. 
Additionally, we assume that $Q$ does not contain explicit time points in 
rules---note that the rules in our running example satisfy this restriction. 
We believe that dropping this assumption does not
affect tractability, but we leave this question open for future work.

By Definition \ref{definition:forgetting}, to decide $\textsc{Forget}$ for $I$,
we must check whether the answers $Q(D \cup U, \tau)$ are included in
$Q(D[\tmem] \cup U, \tau)$ for every $\tin$-update $U$ and $\tau \geq \tout$.
Similarly to the case of $\textsc{DTP}$, we identify two time intervals of
polynomial size in data, and show that it suffices to consider only updates $U$
over the first interval and only time points $\tau$ over the second interval. In
contrast to $\textsc{DTP}$, however, we need to potentially consider all
possible such updates and cannot restrict ourselves to a single critical
one.

Note, however, that checking the aforementioned inclusion of query answers for
all relevant updates and time points would lead to an exponential blowup in
data complexity.  To overcome this, we define instead nonrecursive queries
$Q_1$ and $Q_2$ such that the desired condition holds if and only if
$Q_1 \sqsubseteq Q_2$, where $Q_1$ and $Q_2$ contain a (fixed) rule set derived
from $Q$ and a portion of the history $D$.  Then, we show that checking such
containment where only the data-dependent rules are considered part of the
input is feasible in polynomial time.

We start by identifying the set of relevant time points for query answers and 
updates.

\begin{definition}
    A time point $\tau$ is 
    \emph{output-relevant} for $I$ if $\tout \le \tau \leq \tmem + \radius(Q)$.
    In turn, 
    it is 
    \emph{update-relevant} for $I$ if it satisfies $\tin < \tau \leq \tmem + 2\cdot\radius(Q)$. 
\end{definition}

Intuitively, answers at time points bigger than $\tmem + \radius(Q)$ cannot be
affected by history facts that hold before $\tmem$; thus, we do not need to
consider answers at time points that are not output-relevant.  In turn,
output-relevant time points cannot depend on facts in an update after
$\tmem + 2\cdot\radius(Q)$; thus, it suffices to consider updates containing
only facts that hold at update-relevant time points.  We next construct the
aforementioned queries $Q_1$ and $Q_2$ using the identified update-relevant and
output-relevant time points. 

\begin{definition}\label{def:construction}
    Let $B$ be a fresh temporal IDB unary predicate and let $D_0$ consist of all
    facts $B(\tau)$ for each update-relevant time point $\tau$.  For $\psi$ as in
    Definition~\ref{def:critical-query}, let $\Pi$ be the smallest program
    containing: \emph{(i)} each rule in $\psi(\Pi_Q)$ having a predicate different
    from $P_Q$ in the head; \emph{(ii)} each rule obtained by grounding the time
    argument to an output-relevant time point in a rule in $\psi(\Pi_Q)$ having
    $P_Q$ as head predicate; and \emph{(iii)} the rule
    $P(\vec{x}, t) \land B(t) \to P'(\vec{x}, t)$ for each temporal EDB predicate
    $P$ occurring in $\Pi_Q$, where $P' = \psi(P)$.

    We now let $Q_1 = \langle P_Q, \Pi_1\rangle$ and
    $Q_2 = \langle P_Q, \Pi_2 \rangle$, where
    $\Pi_1 = \Pi \cup D_0 \cup \psi(D)$, and 
    $\Pi_2 = \Pi \cup D_0 \cup \psi(D[\tmem])$.
\end{definition}

Intuitively, the facts about $B$ are used to `tag' the update-relevant time points; 
rule $P(\vec{x}, t) \land B(t) \to P'(\vec{x}, t)$, when applied to the history $D$ and any update $U$, will `project'  $U$ to the relevant time points and filter out
all facts in $D$.  Finally, the rules in \emph{(i)} and \emph{(ii)}
allow us to derive the same consequences (modulo predicate renaming) as $\Pi_Q$, but only over the output-relevant time points.

We can now establish correctness of our approach.

\begin{restatable}{lemma}{nrforgettingcorrectness}
    \label{lem:forgetting-correctness}
    \textsc{Forget} holds for $I$ iff $Q_1 \sqsubseteq Q_2$, where $Q_1$ and
    $Q_2$ are as given in Definition \ref{def:construction}.
\end{restatable}

\begin{restatable}{theorem}{nrforgettingtheorem}
    \label{theorem:forgettingnonrecursive}
    $\textsc{Forget}$ is in \ptime{} in data complexity if restricted to
    nonrecursive \connected{} queries whose rules contain no time points.
\end{restatable}

This concludes our discussion of the data-dependent problems motivated by our
`online' stream reasoning algorithm (recall Algorithm 1). 
In the following section, we turn our attention to the data-independent 
problems motivated by our `offline' approach (Algorithm 2).

\subsection{Data-Independent Problems}\label{sec:delay}

Query containment can be reduced to 
both $\textsc{Delay}$ and $\textsc{Window}$ using a variant of the reduction we used  
in Section~\ref{theorem:forgetting}
for the forgetting problem.
As a result, we can show undecidability
of  both of our data-independent reasoning problems in the general case.

\begin{restatable}{theorem}{undecidabilitydelay}
    \label{theorem:undecidabilitydidelay}
    \textsc{Delay} is undecidable.
\end{restatable}

\begin{restatable}{theorem}{undecidabilitywindow}
    \textsc{Window} is undecidable.
\end{restatable}

Furthermore, our reductions from query containment preserve the shape of the queries and hence
they also provide a $\conexptime$ lower bound for
both problems in the nonrecursive case \cite{benedikt2010impact}.
In the remainder of this section, we show that this bound is tight.

The $\conexptime$ upper bounds are obtained via reductions from
$\textsc{Delay}$ and $\textsc{Window}$ into query containment for temporal
Datalog, which we detail in the remainder of this section.  Similarly to
previous sections, we assume that queries are \connected{} and also that they
do not contain explicit time points or objects; the latter restriction is
consistent with the `purity' assumption in \cite{benedikt2010impact}.

Note, however, that the upper bound in \cite{benedikt2010impact} for query containment
only holds for standard nonrecursive Datalog; therefore, we first establish
that this upper bound extends to the temporal case. Intuitively, temporal
queries can be transformed into nontemporal ones by grounding them to a finite
number of relevant time points based on their (finite) radius.

\begin{restatable}{lemma}{nrtcqinconexptime}
    \label{lem:containment}
    Query containment restricted to nonrecursive queries that are \connected{}
    and constant-free is in $\conexptime$.
\end{restatable}

We now proceed to discussing our reductions from $\textsc{Delay}$ and
$\textsc{Window}$ into temporal query containment.

Consider a fixed, but arbitrary, instance $I = \langle Q,d \rangle$ of
$\textsc{Delay}$.  We construct queries $Q_1$ and $Q_2$ providing the
basis for our reduction.

Let $A$ and $B$ be fresh fresh unary temporal predicates, where $A$ is EDB and $B$ is IDB.
Furthermore, let $G$ be a fresh temporal IDB predicate of the same arity as $P_Q$.
Let $\Pi_1$ extend $\Pi_Q$ with the following rule:
\begin{equation}
    \label{eq:delayruleone}
    P_Q(\vec{x}, t) \land A(t) \to G(\vec{x}, t)
\end{equation}
and let $Q_1 = \langle G, \Pi_1 \rangle$.
Intuitively, $Q_1$   restricts the answers to $Q$ to
time points where $A$ holds.

Let $Q_2 = \langle G, \Pi_2 \rangle$, where $\Pi_2$ is now the program obtained
from $\psi(\Pi_Q)$ by adding the previous rule \eqref{eq:delayruleone} and the
following rules for each $k$ satisfying $-\radius(Q) \leq k \leq d$ and each
temporal EDB predicate $P$ in $\Pi_Q$, where $P' = \psi(P)$:
\begin{align}
    A(t) &\to B(t+k)\label{eq:delayruletwo}\\
    P(\vec{x}, t) \land B(t) &\to P'(\vec{x}, t)\label{eq:delayrulethree}
\end{align}
Intuitively, $Q_2$ further restricts the answers to $Q_1$ at any time $\tout$ to those that can be derived using 
facts in the interval $[\tout -\radius(Q),\tout + d]$.  

It then follows that
$Q_1 \sqsubseteq Q_2$ if and only if $\textsc{Delay}$ holds for $I$. Furthermore, the construction of $Q_1$
and $Q_2$ is feasible in $\logspace$.

\begin{restatable}{theorem}{nrdelaytheorem}
    \label{theorem:didelaynonrecursive}
    \textsc{Delay} restricted to nonrecursive queries that are \connected{} and constant-free is \conexptimecomplete{}.
\end{restatable}

We conclude by providing the upper bound for $\textsc{Window}$. For this,
consider an arbitrary instance $I = \langle Q, d, s \rangle$. Similarly to the
case of $\textsc{Delay}$, we construct queries $Q_1$ and $Q_2$ for which
containment holds iff $\textsc{Window}$ holds for $I$.  Let $\mathit{In}$, $\mathit{Out}$ and $B$
be fresh unary temporal predicates, where $\mathit{In}$ is EDB and $\mathit{Out}$, $B$ are
IDB. Furthermore, as before, let $G$ be a fresh temporal IDB predicate of the
same arity as $P_Q$.

For $j$ a nonnegative
integer, let $\Pi^j$ be the program extending
$\psi(\Pi_Q)$ with the
following rules for each $k$
satisfying 
$-d < k \leq -s+\radius(Q)$,
each $\ell$ satisfying
$-j < \ell \leq -s + 2\cdot\radius(Q)$,
and each temporal EDB predicate $P$ in $\Pi_Q$, where $P' = \psi(P)$:
\begin{align}
    \label{eq:uwindowruletwo} \mathit{In}(t) &\to \mathit{Out}(t+k) \\
    \label{eq:uwindowruleone} P_Q(\vec{x}, t) \land \mathit{Out}(t) &\to G(\vec{x}, t) \\
    \label{eq:uwindowrulethree} \mathit{In}(t) &\to B(t+\ell) \\
    \label{eq:uwindowrulefour} P(\vec{x}, t) \land B(t) &\to P'(\vec{x}, t)
\end{align}
We define $Q_1 = \langle G, \Pi^{d + \radius(Q)} \rangle$ and $Q_2 = \langle G, \Pi^s \rangle$.
Intuitively, 
given a dataset for which $\mathit{In}$ holds for a set of time points $\mathbf{T}$,
query $Q_1$ 
captures the 
answers to $Q$ at time points $\tout$ within the  the interval  $[\tau-d, \tau-s+\radius(Q)]$ for some $\tau \in \mathbf{T}$;
in turn, for each such interval $[\tau-d, \tau-s+\radius(Q)]$, query $Q_2$ further restricts the answers to $Q_1$
to those that depend on input facts holding after $\tau - s$.

Since these queries can again be constructed in $\logspace$, we obtain the desired upper bound.

\begin{restatable}{theorem}{nrwindowtheorem}
    \label{theorem:window}
    \textsc{Window} restricted to nonrecursive queries that are \connected{} and constant-free is \conexptimecomplete{}.
\end{restatable}

\section{Related Work}

The main challenges posed by stream processing and 
the basic architecture of a stream management system 
were first discussed in \cite{babu2001continuous,babcock2002models}.
\citeauthor{arasu2006cql}~(\citeyear{arasu2006cql})
proposed the CQL query language, which 
extends SQL with  a notion of \emph{window}---a mechanism 
that allows one to reduce stream processing to traditional query evaluation. 
Since then, there have been numerous extensions and variants of CQL, which include 
a number of  stream query languages for the Semantic Web \cite{barbieri2009csparql,le2011native,le2013elastic,dell2015towards}.

In recent years, there have been several proposals for a general-purpose rule-based language in the context of stream reasoning.
Streamlog \cite{zaniolo2012streamlog} is a temporal Datalog language, which differs from the language considered in our paper in that it provides
nonmonotonic negation and restricts the syntax so that only facts over time points explicitly present in the data can be derived. Furthermore, the 
focus in \cite{zaniolo2012streamlog} is on dealing with so-called `blocking queries', which are those whose answers may depend on 
input facts arbitrarily far in the future; for this, a syntactic fragment of the language is provided that precludes blocking queries.
LARS is a temporal rule-based language featuring window constructs
and negation interpreted according to the stable model semantics \cite{beck2015lars,beck2015answer,beck2016equivalent}.
The semantics of LARS is rather different from that of temporal Datalog; in particular, 
the number of time points 
in a model is  considered as part of the input to query evaluation, and hence is restricted to be finite; 
furthermore, the notion of window  is built-in in LARS.

Stream reasoning has been 
studied in the context of 
RDF-Schema \cite{barbieri2010incremental},
and ontology-based data access \cite{calbimonte2010enabling,ozcep2014stream}.  
In these works, the input data is assumed to arrive as a stream, but 
the ontology language is assumed to be nontemporal. 
Stream reasoning has also been considered in the unrelated context of 
complex event processing \cite{anicic2011ep,dao2015enriching}.

There have been a number of proposals for rule-based languages in the context of 
temporal reasoning; here, the focus is on query evaluation over static temporal data, rather than on 
reasoning problems that are specific to stream processing. 
Our temporal Datalog language is a notational variant of
Datalog$_\mathrm{1S}$---the core language for temporal deductive databases
\cite{chomicki1988temporal,chomicki1989relational,chomicki1990polynomial}.
Templog is an extension of Datalog with modal temporal operators \cite{abadi1989temporal}, which was 
shown to be captured by Datalog$_\mathrm{1S}$ \cite{baudinet1992temporal}.
Datalog was extended with integer periodicity and gap-order constraints in \cite{toman1998datalog}; such constraints
allow for the representation of infinite periodic phenomena.
Finally, DatalogMTL is a recent Datalog extension based on metric temporal logic \cite{brandt2017ontology}.

In the setting of database constraint checking, a problem related to our window 
problem was considered by 
\citeauthor{chomicki1995encoding}~\shortcite{chomicki1995encoding}, 
who obtained some positive results for queries formulated in temporal 
first-order logic.

\section{Conclusion and Future Work}

In this paper, we have proposed novel decision problems relevant to the design of 
stream reasoning algorithms, and have studied their computational properties for temporal Datalog.
These problems capture the key challenges behind rule-based stream reasoning, where 
rules can propagate information both to past and future time points. 
Our results suggest that rule-based stream reasoning is feasible in practice for 
nonrecursive temporal Datalog queries. Our problems are, however, either intractable in data complexity or
undecidable in the general case. 

We have made several mild technical assumptions in our upper bounds for nonrecursive queries, 
which we plan to lift in future work.
Furthermore, we have assumed throughout the paper that numbers in the input are encoded in unary;  we are currently looking into 
the impact of binary encoding on the complexity of our problems.
Finally, we are planning to study extensions of nonrecursive temporal Datalog for which 
decidability of all our problems can be ensured.

\section*{Acknowledgments}

This research was supported by 
the SIRIUS Centre for Scalable Data Access in the Oil and Gas Domain, 
the Royal Society, 
and the EPSRC projects DBOnto, MaSI$^3$, and ED$^3$.

\bibliographystyle{aaai}
\bibliography{bibliography}

\ifextendedversion
    \appendix
    \newpage
\onecolumn

\begin{definition}
    A \emph{derivation} of a fact $\alpha$ from a program $\Pi$ is a
    finite labelled tree such that:
    \begin{inparaenum}[\it(i)]
    \item each node is labelled with a ground instance of a rule in $\Pi$;
    \item fact $\alpha$ is the head of the rule labelling the root;
    \item if the rule of a node $w$ has a non-empty body containing atoms 
        $\alpha_1, \dots, \alpha_m$, then 
        $w$ has $m$ children and $\alpha_i$ is the head of the rule labelling the $i$-th child.
    \end{inparaenum}
\end{definition}

\section{Proofs for Section~\ref{subsec:dtp}}

\dtpcriticaldomain*
\begin{proof}
    If \textsc{DTP} holds for $I$, then the condition of the proposition clearly
    holds as well.  For the converse, assume that $\textsc{DTP}$ does not hold for
    $I$ and hence there exists a tuple $\vec{c}$ of objects and a $\tin$-update
    $U$ such that $\vec{c} \in Q(D \cup U, \tout)$ and
    $\vec{c} \notin Q(D, \tout)$.  Let $h$ be the function mapping every object
    $o \in \criticaldomain{I}$ to itself and every other object to $o_I$, let
    $\vec{o} = h(\vec{c})$, and let $V = h(U)$.  Clearly,
    $\vec{o} \in Q(D \cup V, \tout)$, where $\vec{o}$ is a tuple over
    $\criticaldomain{I}$, and $V$ is a $\tin$-update involving only objects in
    $\criticaldomain{I}$.  It remains to show that $\vec{o} \notin Q(D, \tout)$.
    If $o_I \notin \vec{o}$, then $\vec{o} = \vec{c}$ and
    $\vec{c} \notin Q(D, \tout)$ holds by our assumption.  Otherwise,
    $\vec{o} \notin Q(D, \tout)$ holds because $o_I$ does not occur in
    $\Pi_Q \cup D$.
\end{proof}

\dtpupper*
\begin{proof}
    Assume that \textsc{DTP} holds for $I$. We show that
    $\vec{o} \in \criticalquery{I}(D \cup \criticalupdate{I}, \tout)$ implies
    $\vec{o} \in \criticalquery{I}(D, \tout)$ for every tuple $\vec{o}$ over
    $\criticaldomain{I}$.  Let $\vec{o}$ be a tuple over $\criticaldomain{I}$
    such that $\vec{o} \in \criticalquery{I}(D \cup \criticalupdate{I}, \tout)$.
    Let $\delta$ be a derivation of $P_Q(\vec{o}, \tout)$ from 
    $\Pi_{\criticalquery{I}} \cup D \cup \criticalupdate{I}$.  Let $\delta'$ be
    the derivation obtained from $\delta$ by first removing each node labelled by
    an instance of any of the
    additional rules introduced in Definition~\ref{def:critical-query}
    and
    then replacing each $P'$ with its corresponding EDB predicate $P$---i.e., 
    the predicate $P$ such that $P' = \psi(P)$.  
    Let $U$ be the $\tin$-update consisting of each temporal EDB fact labelling a leaf of
    $\delta'$ and having time argument strictly bigger than $\tin$. Then, by the
    construction of $\criticalquery{I}$, $\delta'$ is a derivation of 
    $P_Q(\vec{o}, \tout)$ from $\Pi_Q \cup D \cup U$, and hence
    $\vec{o} \in Q(D \cup U, \tout)$.  Therefore, $\vec{o} \in Q(D, \tout)$ by
    Proposition~\ref{prop:delay} because \textsc{DTP} holds for $I$ by
    assumption, and hence $\vec{o} \in \criticalquery{I}(D, \tout)$ by
    the previously observed properties 
    of $\criticalquery{I}$.

    For the converse, assume that
    $\vec{o} \in \criticalquery{I}(D \cup \criticalupdate{I}, \tout)$ implies
    $\vec{o} \in \criticalquery{I}(D, \tout)$ for every tuple $\vec{o}$ over
    $\criticaldomain{I}$. We prove that \textsc{DTP} holds for $I$ using
    Proposition~\ref{prop:delay}, by showing that
    $\vec{o} \in Q(D \cup U, \tout)$ implies $\vec{o} \in Q(D, \tout)$ for every
    tuple $\vec{o}$ over $\criticaldomain{I}$ and $\tin$-update $U$ involving
    only objects of $\criticaldomain{I}$. Let $\vec{o}$ be such a tuple and $U$
    such a $\tin$-update, and suppose $\vec{o} \in Q(D \cup U, \tout)$.
    By the previously observed properties of $\criticalquery{I}$ and
    $\criticalupdate{I}$, we then have
    $\vec{o} \in \criticalquery{I}(D \cup \criticalupdate{I}, \tout)$, and hence
    $\vec{o} \in \criticalquery{I}(D, \tout)$ by assumption.  Therefore,
    $\vec{o} \in Q(D, \tout)$ by the properties 
    of $\criticalquery{I}$.
\end{proof}

\begin{lemma}
    \label{lem:dtp-lower}
    There exists a \logspace{}-computable many-one reduction $\phi$ from query
    evaluation to $\textsc{DTP}$ such that, for each instance
    $I = \langle Q, D, \vec{a} \rangle$ of query evaluation, the query $Q'$ in
    $\phi(I)$ is independent of $D$ and $\vec{a}$.
\end{lemma}

\begin{proof}
    Let $I = \langle Q, D, \vec{a} \rangle$ be an instance of query evaluation.
    We assume w.l.o.g.\ that $Q$ is temporal and hence
    $\vec{a} = \langle \vec{o}, \tau \rangle$ with $\vec{o}$ a tuple of
    objects---otherwise, simply consider tuple $\langle \vec{a}, 0 \rangle$
    instead of $\vec{a}$ and query $ \langle P, \Pi \rangle$ with
    $\Pi = \Pi_Q \cup \{P_Q(\vec{x}) \rightarrow P(\vec{x}, 0)\}$ instead of $Q$.

    We now define the instance $\phi(I)$ of $\textsc{DTP}$ corresponding to
    $I$.  Let $\edb{T}$ and $\edb{A}$ be fresh temporal predicates, where
    $\edb{T}$ is EDB and unary and $\edb{A}$ is EDB and of the same arity as
    $P_Q$.  Let $D' = D \cup \{\edb{A}(\vec{o}, \tau)\}$ and let $Q' $ of the
    same arity as $Q$ where $\Pi_{Q'}$ is $\Pi_Q$ extended with the following
    rules:
    \begin{align*}
        \edb{T}(t+1) \wedge \edb{A}(\vec{x}, t) \rightarrow \alignhere
        P_{Q'}(\vec{x},t)
        \\
        P_Q(\vec{x}, t) \wedge \edb{A}(\vec{x}, t) \rightarrow \alignhere
        P_{Q'}(\vec{x},t)
    \end{align*}
    We argue that $\vec{o} \in Q(D, \tau)$ if and only if $\textsc{DTP}$ holds
    for $\phi(I) = \langle Q', D', \tau, \tau \rangle$.  If
    $\vec{o} \in Q(D, \tau)$, we show that
    $Q'(D' \cup U, \tau) \subseteq Q'(D',\tau)$ for every $\tau$-update $U$ and
    hence $\textsc{DTP}$ holds for $\phi(I)$.  Assume that
    $\vec{c} \in Q'(D' \cup U, \tau)$ for some $\tau$-update $U$. Then, since
    $P_{Q'}(\vec{c},\tau)$ can only be entailed by one of the two new rules in
    $\Pi_{Q'}$, dataset $D' \cup U$ must contain the fact $A(\vec{c},\tau)$;
    note, however, that this fact cannot be contained in $U$ because $U$ is a
    $\tau$-update, and it is also not in $D$ because it mentions $A$. Therefore,
    $\vec{c} = \vec{o}$; but now, by the assumption that $\vec{o} \in Q(D, \tau)$
    and by the construction of $Q'$ and $D'$, we have $\vec{o} \in Q'(D', \tau)$,
    as required.

    Next, assume that $\textsc{DTP}$ holds for
    $\phi(I) = \langle Q', D', \tau, \tau \rangle$.  We show that
    $\vec{o} \in Q(D, \tau)$.  Consider the $\tau$-update $U$ containing the fact
    $T(\tau + 1)$.  Then, $\vec{o} \in Q'(D' \cup U, \tau)$ because
    $A(\vec{o}, \tau) \in D'$, and hence $\vec{o} \in Q'(D', \tau)$ because
    $\textsc{DTP}$ holds for $\phi(I)$ by assumption.  We have that
    $\vec{o} \in Q'(D', \tau)$ implies $\vec{o} \in Q(D', \tau)$ because
    $T(\tau+1) \notin D'$ since $T$ is fresh.  Therefore,
    $\vec{o} \in Q(D, \tau)$ because $D' \setminus D = \{ A(\vec{o}, \tau) \}$
    and $A$ does not occur in $Q$.
\end{proof}

\dtptheorem*
\begin{proof}
    Hardness follows by Lemma~\ref{lem:dtp-lower},
    since query evaluation is \pspacecomplete{} in data complexity 
    by the results in \cite{chomicki1988temporal}.

    We show an algorithm that decides \textsc{DTP} on $I = \langle Q, D, \tin, \tout \rangle$ 
    in polynomial space if the query $Q$ is considered fixed.
    According to Lemma~\ref{lem:dtp-upper}, 
    it is sufficient to iterate over all tuples $\vec{o}$ of objects from $\criticaldomain{I}$,
    rejecting if 
    $\vec{o} \in \criticalquery{I}(D \cup \criticalupdate{I}, \tout)$ and
    $\vec{o} \notin \criticalquery{I}(D, \tout)$,
    and accepting if we can complete all the iterations without rejecting.
    Let $a$ be the maximum arity of a predicate in $Q$,  
    let $c$ be the number of objects in $\Pi_Q \cup D$,
    and let $p$ be the number of predicates in $Q$.
    Note that, with respect to the size of the input,
    $a$ and $p$ are constant and $c$ is linear.
    We can build $\criticalquery{I}$ in constant time because $\criticalquery{I}$ depends only on $Q$,
    and we can build $\criticalupdate{I}$ in polynomial time because the number of facts in 
    $\criticalupdate{I}$ is at most $1 + p \cdot (c+1)^a$.
    The number of iterations is polynomial because the number relevant object tuples is $(c+1)^a$.
    Finally, note that we can check both $\vec{o} \in \criticalquery{I}(D \cup \criticalupdate{I}, \tout)$ and 
    $\vec{o} \notin \criticalquery{I}(D \cup \criticalupdate{I}, \tout)$ in polynomial space,
    since query evaluation and its complement are \pspacecomplete{} in data complexity 
    by the results in \cite{chomicki1988temporal}.
\end{proof}

\subsubsection{Nonrecursive case}

\nrdtpcriticalupdate*
\begin{proof}
    We prove the claim by induction on the rank of $P$. We assume w.l.o.g.\ that
    $U$ contains only predicates in $\Pi_{Q}$.

    In the base case $\rank(P) = 0$.  Let
    $\Pi_Q \cup D \cup U \models P(\vec{o}, \tau)$ such that
    $|\tau - \tau'|\le a \cdot (\rank(\Pi_Q)-\rank(P))$ for some
    $\tau'\in\mathbf{T}$.  We show
    $\Pi_Q \cup D \cup \nrcriticalupdate{I} \models P(\vec{o}, \tau)$.  Since
    $\rank(P)=0$, we have $P(\vec{o}, \tau) \in \Pi_Q \cup D \cup U$ because $P$
    occurs only in facts. Moreover, we have
    $|\tau - \tau'|\leq a \cdot \rank(\Pi_Q)\le\radius(\Pi_Q)$. We distinguish two
    cases. If $\tau\le\tin$, then $P(\vec{o},\tau)\in\Pi_Q\cup D$ since $U$ only
    contains facts with time points after $\tin$, and the claim
    follows. Otherwise, we have $\tin<\tau\le\tau'+\radius(\Pi_Q)$, and hence
    $\tau$ is critical. Since $\nrcriticalupdate{I}$ contains all facts involving
    only EDB predicates in $\Pi_Q$, objects in $\criticaldomain{I}$, and critical
    time points, we then have $P(\vec{o},\tau)\in\nrcriticalupdate{I}$, and the
    claim follows.

    For the inductive step, we assume that the claim holds for every predicate of
    rank at most $n$ and show it for $\rank(P)=n+1$. Let
    $\Pi_Q \cup D \cup U \models P(\vec{o}, \tau)$ such that
    $|\tau-\tau'|\le a\cdot(\rank(\Pi_Q)-n-1)$ for some $\tau'\in\mathbf{T}$.
    Let $\delta$ be a derivation of $P(\vec{o}, \tau)$ from
    $\Pi_Q \cup D \cup U$.  We show
    $\Pi_Q \cup D \cup \nrcriticalupdate{I} \models P(\vec{o}, \tau)$.  Let $r$
    be the label of the root of $\delta$, and let $r'$ be a rule in $\Pi_Q$ such
    that $r$ is an instance of $r'$. It suffices to show that
    $\Pi_Q \cup D \cup \nrcriticalupdate{I} \models \alpha$ for each atom
    $\alpha\in\body(r)$. Let $\alpha$ be an arbitrary such atom. 
    We have $\alpha=P_1(\vec{o}_1,\tau_1)$ for $\tau_1$ a time
    term, by our assumption that $Q$ does not contain rigid atoms.
    Since $\delta$ is a derivation, we have $\Pi_Q \cup D \cup U \models \alpha$.
    We distinguish two subcases.

    If the atom corresponding to $\alpha$ in $r'$ mentions a time variable, so does its head because $\Pi_Q$ is \connected{}, and hence we have
    $|\tau_1-\tau|\le a$. Consequently, since
    $|\tau-\tau'|\le a\cdot(\rank(\Pi_Q)-n-1)$, we have
    $|\tau_1-\tau'|\le a\cdot(\rank(\Pi_Q)-n)\le a\cdot(\rank(\Pi_Q)-\rank(P_1))$;
    $\Pi_Q \cup D \cup \nrcriticalupdate{I} \models \alpha$ then follows from
    $\Pi_Q \cup D \cup U \models \alpha$ by the inductive hypothesis.

    If the atom corresponding to $\alpha$ in $r'$ mentions no time variable,
    $\tau_1$ must be a time point, and hence $\tau_1\in\mathbf{T}$. Clearly,
    $|\tau_1-\tau_1|=0\le a\cdot(\rank(\Pi_Q)-n)$, and hence
    $\Pi_Q \cup D \cup \nrcriticalupdate{I} \models \alpha$ follows from
    $\Pi_Q \cup D \cup U \models \alpha$ by the inductive hypothesis.
\end{proof}

\nrdtpmain*
\begin{proof}
    If $\textsc{DTP}$ holds for $I$, then trivially
    $\vec{o} \in Q(D \cup \nrcriticalupdate{I}, \tout)$ implies
    $\vec{o} \in Q(D, \tout)$ for every tuple $\vec{o}$ over
    $\criticaldomain{I}$, because $\nrcriticalupdate{I}$ is a $\tin$-update.  For
    the converse, assume that $\vec{o} \in Q(D \cup \nrcriticalupdate{I}, \tout)$
    implies $\vec{o} \in Q(D, \tout)$ for every tuple $\vec{o}$ over
    $\criticaldomain{I}$.  We prove that $\textsc{DTP}$ holds for $I$ by showing
    that $\vec{o} \in Q(D \cup U, \tout)$ implies $\vec{o} \in Q(D, \tout)$ for
    every tuple $\vec{o}$ over $\criticaldomain{I}$ and $\tin$-update $U$
    involving only objects of $\criticaldomain{I}$; the claim then holds by
    Proposition~\ref{prop:delay}.  Let $\vec{o}$ be a tuple over
    $\criticaldomain{I}$ and let $U$ be a $\tin$-update involving only objects of
    $\criticaldomain{I}$ such that $\vec{o} \in Q(D \cup U, \tout)$. Since
    $\tout\in\mathbf{T}$ and $|\tout-\tout|=0$, by Lemma~\ref{lem:nr-basic} for
    $\tau=\tau'=\tout$ it then follows that
    $\vec{o} \in Q(D \cup \criticalupdate{I}, \tout)$.  By our assumption,
    $\vec{o} \in Q(D, \tout)$.
\end{proof}

\nrdtptheorem*
\begin{proof}
    We show an algorithm that decides \textsc{DTP} on $I = \langle Q, D, \tin, \tout \rangle$ 
    in polynomial time if the query $Q$ is considered fixed.
    According to Lemma~\ref{lem:nonrec-correct}, 
    it is sufficient to iterate over all tuples $\vec{o}$ of objects from $\criticaldomain{I}$,
    rejecting if 
    $\vec{o} \in Q(D \cup \nrcriticalupdate{I}, \tout)$ and
    $\vec{o} \notin Q(D, \tout)$,
    and accepting if we can complete all the iterations without rejecting.
    Let $a$ be the maximum arity of a predicate in $Q$,  
    let $c$ be the number of objects in $\Pi_Q \cup D$,
    and let $p$ be the number of predicates in $Q$.
    Note that, with respect to the size of the input, 
    the values $a$, $p$ and $\radius(Q)$ are constant;
    furthermore, $\tau_0 - \tin$ and $c$ are linear.
    We can build $\nrcriticalupdate{I}$ in polynomial time because the number of facts in 
    $\nrcriticalupdate{I}$ is bounded by $p \cdot (\radius(Q) + \tau_0 - \tin) \cdot (c+1)^a$.
    The number of iterations is polynomial because the number of relevant object tuples is $c^a$.
    Finally, checking both
    $\vec{o} \in Q(D \cup \nrcriticalupdate{I}, \tout)$ and
    $\vec{o} \notin Q(D, \tout)$ is in \aczero{}.
\end{proof}

\section{Proofs for Section~\ref{subsec:hardness-forgetting}}

\begin{lemma} \label{lemma:hardnessforgetting} %
    There exists a \logspace{}-computable many-one reduction $\phi$ from datalog
    query containment to $\textsc{Forget}$ such that, for every instance
    $I = \langle Q_1, Q_2 \rangle$ of datalog query containment, the query in
    $\phi(I)$ is nonrecursive if $Q_1$ and $Q_2$ are nonrecursive.
\end{lemma}
\begin{proof}
    Let $\langle Q_1, Q_2 \rangle$ be an instance of query containment with $Q_1$
    and $Q_2$ datalog queries. Without loss of generality,
    $P_{Q_1} = P_{Q_2} = G$.  For any rigid $n$-ary IDB predicate $P$ and
    $i \in \{ 1,2 \}$, let $P_i$ be a fresh rigid $n$-ary IDB predicate uniquely
    associated with $P$ and $i$.  For any rigid $n$-ary EDB (resp., IDB)
    predicate $P$, let $P^\mathrm{t}$ be a fresh temporal $(n+1)$-ary EDB (IDB)
    predicate uniquely associated with $P$.  Let $t$ be a time variable.  For
    $i \in \{ 1,2 \}$, let $\program_i$ be $\program_{Q_i}$ after replacing each
    rigid $n$-ary IDB predicate $P$ with $P_i$; let $\program_i'$ be
    $\program_i$ after replacing each rigid atom $P(\vec{u})$ with the temporal
    atom $P^\mathrm{t}(\vec{u}, t)$.  Let $A$ be a fresh temporal unary EDB
    predicate.  Let $Q$ be the query such that $P_Q$ is a fresh temporal IDB
    predicate of the same arity as $G_1^\mathrm{t}$ (or, equivalently, as
    $G_2^\mathrm{t}$), and $\program_Q$ is $\program_1' \cup \program_2'$
    extended with the following rules:
    \begin{align}
        \label{eq:forgetruleone} A(t-1) \wedge G_1^\mathrm{t}(\vec{x},t) \rightarrow \alignhere P_Q(\vec{x},t) \\
        \label{eq:forgetruletwo} G_2^\mathrm{t}(\vec{x},t) \rightarrow \alignhere P_Q(\vec{x},t) 
    \end{align}
    Clearly, $Q$ can be constructed in logarithmic space w.r.t.\ the size of
    $Q_1$ and $Q_2$.

    Let $D = \{ A(0) \}$.  We show that $Q_1 \sqsubseteq Q_2$ iff
    $\textsc{Forget}$ holds for
    $\phi(\langle Q_1, Q_2 \rangle) = \langle Q, D, 1, 1, 0 \rangle$.

    Assume that $Q_1 \sqsubseteq Q_2$ holds.  We show that $\textsc{Forget}$
    holds for $\langle Q, D, 1, 1, 0 \rangle$, by showing that
    $Q(D \cup U, \tau)\subseteq Q(D[0] \cup U, \tau)$ for every $1$-update $U$
    and time point $\tau \geq 1$.  Let $\vec{o}$ be a tuple of objects, let $U$
    be a $1$-update and let $\tau \geq 1$ such that
    $\vec{o} \in Q(D \cup U, \tau)$.  Let $\delta$ be a derivation of
    $P_Q(\vec{o}, \tau)$ from $\Pi_Q \cup D \cup U$.  The root of $\delta$ is
    labelled with an instance of either rule~\eqref{eq:forgetruleone} or
    rule~\eqref{eq:forgetruletwo} since $P_Q$ does not occur in
    $\program_1' \cup \program_2'$.  We consider the two cases separately.  If
    the label is an instance of rule \eqref{eq:forgetruleone}, then
    $\Pi_Q \cup D \cup U \models G_1^\mathrm{t}(\vec{o},\tau)$; then
    $\Pi_1' \cup D \cup U \models G_1^\mathrm{t}(\vec{o},\tau)$ by the
    construction of $\Pi_Q$; then
    $\Pi_1' \cup U \models G_1^\mathrm{t}(\vec{o},\tau)$ because facts in
    $\Pi_1'$ all have $t$ as time argument, and $\tau$ satisfies $\tau \geq 1$
    by assumption; then $\vec{o} \in Q_1(U')$ by the construction of $\Pi_1'$,
    where $U'$ is the dataset consisting of each rigid fact $P(\vec{c})$ for
    $P^\mathrm{t}(\vec{c}, \tau) \in U$; then $\vec{o} \in Q_2(U')$ since
    $Q_1\sqsubseteq Q_2$ by assumption; then
    $\Pi_2' \cup U \models G_2^\mathrm{t}(\vec{o},\tau)$ by the construction of
    $\Pi_2'$; then $\Pi_Q \cup U \models G_2^\mathrm{t}(\vec{o},\tau)$ because
    $\Pi_2' \subseteq \Pi_Q$; then $\vec{o} \in Q(U, \tau)$ by
    rule~\eqref{eq:forgetruletwo}, and hence $\vec{o} \in Q(D[0] \cup U, \tau)$
    since $D[0] = \emptyset$.  If the root of $\delta$ is labelled with an
    instance of rule \eqref{eq:forgetruletwo}, then
    $\Pi_Q \cup D \cup U \models G_2^\mathrm{t}(\vec{o}, \tau)$; then
    $\Pi_2' \cup D \cup U \models G_2^\mathrm{t}(\vec{o}, \tau)$ by the
    construction of $\Pi_Q$; then
    $\Pi_2' \cup U \models G_2^\mathrm{t}(\vec{o}, \tau)$ because facts in
    $\Pi_2'$ all have $t$ as time argument and $\tau$ satisfies $\tau \geq 1$
    by assumption; then $\Pi_Q \cup U \models G_2^\mathrm{t}(\vec{o}, \tau)$
    because $\Pi_2' \subseteq \Pi_Q$, and so
    $\Pi_Q \cup D[0] \cup U \models G_2^\mathrm{t}(\vec{o}, \tau)$ because
    $D[0] = \emptyset$.

    For the converse, assume that $\textsc{Forget}$ holds for
    $\langle Q, D, 1, 1, 0 \rangle$, and hence
    $Q(D \cup U, \tau)\subseteq Q(D[0] \cup U, \tau)$ for every $1$-update $U$
    and time point $\tau \geq 1$.  We show $Q_1(D')\subseteq Q_2(D')$ for every
    dataset $D'$.  Let $\vec{o}$ be a tuple of objects and $D'$ a dataset such
    that $\vec{o} \in Q_1(D')$.  Let $U$ be the $1$-update consisting of each
    fact $P^\mathrm{t}(\vec{c}, 2)$ for $P(\vec{c}) \in D'$.  Then,
    $\Pi_1' \cup U \models G_1^\mathrm{t}(\vec{o},2)$ by the construction of
    $\Pi_1'$; then $\Pi_Q \cup U \models G_1^\mathrm{t}(\vec{o},2)$ because
    $\Pi_1' \subseteq \Pi_Q$, and hence $\vec{o} \in Q(D \cup U, 2)$ by rule
    \eqref{eq:forgetruleone}.  It follows that $\vec{o} \in Q(D[0] \cup U, 2)$
    by assumption, and hence $\vec{o} \in Q(U, 2)$ because $D[0] = \emptyset$.
    Then the root of every derivation of $P_Q(\vec{o}, 2)$ from $\Pi_Q \cup U$
    must be an instance of rule \eqref{eq:forgetruletwo}.  Therefore,
    $\Pi_Q \cup U \models G_2^\mathrm{t}(\vec{o}, 2)$; then
    $\Pi_2' \cup U \models G_2^\mathrm{t}(\vec{o}, 2)$ by the construction of
    $\Pi_Q$, and hence $\vec{o} \in Q_2(D')$ by the construction of $\Pi_2'$
    and $U$.
\end{proof}

\forgettingtheorem*
\begin{proof}
    The claim follows by Lemma~\ref{lemma:hardnessforgetting} since query
    containment for datalog is undecidable by the results
    in~\cite{shmueli1993equivalence}.
\end{proof}

\subsubsection{Nonrecursive Case}

\begin{lemma} \label{lem:nr-forget-crit-tp} %
    Let $a$ be the maximum radius of a rule in $\Pi_Q$. For each dataset $D'$,
    predicate $P$, objects $\vec{o}$, time point $\tau$ and set $\mathbf{T}$
    containing $\tau$ and each time point in $\Pi_Q$,
    $\Pi_Q\cup D'\models P(\vec{o},\tau)$ implies
    $\Pi_Q\cup D'[\min(\mathbf{T})-a\cdot\rank(P)-1]\models P(\vec{o},\tau)$.
\end{lemma}
\begin{proof}
    We proceed by induction on $\rank(P)$. For the base case, let $\rank(P)=0$,
    and suppose $\Pi_Q\cup D'\models P(\vec{o},\tau)$ for some $D'$, $\vec{o}$
    and $\tau$. Since $\rank(P)=0$, $P$ must be EDB (otherwise, facts involving
    $P$ cannot be entailed by $\Pi_Q\cup D'$ as $P$ does not occur in rule heads
    in $\Pi_Q$ and may not occur in $D'$), and hence
    $P(\vec{o},\tau)\in\Pi_Q\cup D'$. But then
    $P(\vec{o},\tau)\in\Pi_Q\cup D'[\min(\mathbf{T})-1]$ since
    $\min(\mathbf{T})\le\tau$; consequently,
    $\Pi_Q\cup D'[\min(\mathbf{T})-1]\models P(\vec{o},\tau)$, as
    required.

    For the inductive step, suppose the claim holds for all predicates of rank at
    most $n$, and let $\Pi_Q\cup D'\models P(\vec{o},\tau)$ where
    $\rank(P)=n+1$. We show
    $\Pi_Q\cup D'[\min(\mathbf{T})-a\cdot(n+1)-1]\models P(\vec{o},\tau)$. Let
    $\delta$ be a derivation of $P(\vec{o},\tau)$ from $\Pi_Q\cup D'$ whose root
    is labelled with an instance $r$ of a rule in $\Pi_Q$. Then, for each
    temporal body atom $\alpha=P'(\vec{o}',\tau')$ of $r$, we have
    $\Pi_Q\cup D'\models\alpha$. Hence, by the inductive hypothesis, for each
    such $\alpha$ we have
    $\Pi_Q\cup D'[\min(\mathbf{T}_\alpha)-a\cdot n_\alpha-1]\models\alpha$, where
    $\mathbf{T}_\alpha$ is the set consisting of $\tau'$ and each time point in
    $\Pi_Q$, and $n_\alpha=\rank(P')\le n$. Note that $\tau'$ is either a time
    point in $\mathbf{T}$ or $\tau'\ge\tau-a$; thus,
    $\min(\mathbf{T}_\alpha)\ge\min(\mathbf{T})-a$, and hence
    $\min(\mathbf{T}_\alpha)-a\cdot
    n_\alpha-1\ge\min(\mathbf{T})-a\cdot(n_\alpha+1)-1\ge\min(\mathbf{T})-a\cdot(n+1)-1$,
    where the last inequality holds since $n_\alpha\le n$. Consequently, for each
    $\alpha$,
    $D'[\min(\mathbf{T}_\alpha)-a\cdot n_\alpha-1]\subseteq
    D'[\min(\mathbf{T})-a\cdot (n+1)-1]$, and hence
    $\Pi_Q\cup D'[\min(\mathbf{T})-a\cdot (n+1)-1]\models\alpha$ by monotonicity
    of entailment. On the other hand, for all rigid body atoms $\beta$ in $r$,
    $\Pi_Q\cup D'\models\beta$ implies
    $\Pi_Q\cup D'[\min(\mathbf{T})-a\cdot(n+1)-1]\models\beta$ since $\Pi_Q$ is \connected{} and hence the validity of $\beta$ does not depend on temporal facts.  
    The claim then follows by $r$.
\end{proof}

\begin{lemma} \label{lem:small-updates} %
    Let $a$ be the maximum radius of a rule in $\Pi_Q$ and let $\mathbf{T}$ be
    the set of time points in $\Pi_Q$.  For each time point $\tau_a$, each
    dataset $D'$, and each time point
    $\tau \leq \tau_0 + a \cdot (\rank(\Pi_Q) - \rank(P))$, where $\tau_0$ is the
    maximum among $\tau_a$ and the time points in $\mathbf{T}$,
    $\Pi_Q \cup D' \models P(\vec{o}, \tau)$ implies
    $\Pi_Q \cup D'' \models P(\vec{o}, \tau)$, where $D''$ consists of each fact
    in $D'$ with time argument $\tau'$ satisfying
    $\tau' \leq \tau_0 + \radius(\Pi_Q)$.
\end{lemma}
\begin{proof}
    Assume $\Pi_Q \cup D' \models P(\vec{o}, \tau)$ for each
    $\tau \leq \tau_0 + a \cdot (\rank(\Pi_Q) - \rank(P))$.  We prove
    $\Pi_Q \cup D'' \models P(\vec{o}, \tau)$ by induction on the rank of $P$.

    In the base case, $\rank(P) = 0$.  Since $P$ occurs only in facts,
    $P(\vec{o}, \tau) \in \Pi_Q \cup D'$.  If $\tau \in \mathbf{T}$, then
    $P(\vec{o}, \tau) \in \Pi_Q$, and hence
    $P(\vec{o}, \tau) \in \Pi_Q \cup D''$.  Otherwise, $P(\vec{o}, \tau) \in D'$
    and $\tau \leq \tau_0 + a\cdot\rank(\Pi_Q)\le\tau_0 + \radius(\Pi_Q)$, and
    hence $P(\vec{o}, \tau) \in \Pi_Q \cup D''$ by the definition of $D''$. In
    either case, the claim follows.

    For the inductive step, we assume that the claim holds for every predicate of
    rank at most $n$ and we show it for $\rank(P) = n + 1$.  Let $\delta$ be a
    derivation of $P(\vec{o}, \tau)$ from $\Pi_Q \cup D'$, let $r$ label the root
    of $\delta$, and let $r'$ be a rule in $\Pi_Q$ such that $r$ is an instance
    of $r'$.  It suffices to show $\Pi_Q \cup D'' \models \alpha$ for each atom
    $\alpha \in \body(r)$.  Let $\alpha$ be an arbitrary such atom.  Since
    $\delta$ is a derivation, $\Pi_Q \cup D' \models \alpha$. We distinguish two
    cases.

    If $\alpha$ is rigid, the claim follows from $\Pi_Q \cup D' \models \alpha$
    since $D'$ and $D''$ coincide on rigid facts and the
    validity of $\alpha$ depends only on rigid facts 
    because $\Pi_Q$ is \connected{}.

    Otherwise, we have $\alpha = P_1(\vec{o}_1, \tau_1)$. If
    $\tau_1 \in \mathbf{T}$, then
    $\tau_1 \leq \tau_0\leq\tau_0+a\cdot(\rank(\Pi_Q)-\rank(P_1))$; then
    $\Pi_Q \cup D'' \models \alpha$ by the inductive hypothesis.  Now, let
    $\tau_1 \notin \mathbf{T}$.  Then the atom corresponding to $\alpha$ in $r'$
    mentions a time variable $t$ and $\head(r')$ mentions the same variable $t$
    because $r'$ is \connected{}; hence, we have
    $\vert \tau_1 - \tau \vert \leq a$, and thus
    $\tau_1 \leq \tau_0 + a \cdot (\rank(\Pi_Q) - n)\le \tau_0 + a \cdot
    (\rank(\Pi_Q) - \rank(P_1))$.  Therefore, $\Pi_Q \cup D'' \models \alpha$ by
    the inductive hypothesis.
\end{proof}

\nrforgettingcorrectness*
\begin{proof}[Proof sketch]
    Assume $Q_1 \sqsubseteq Q_2$.  We prove that \textsc{Forget} holds
    for $I$ by showing that $Q(D \cup U, \tau)\subseteq Q(D[\tmem] \cup
    U, \tau)$ for every $\tin$-update $U$ and time point $\tau \geq
    \tout$.
    Note that $Q$ is \connected{} and constant-free and, hence, 
    so are $Q_1$ and $Q_2$.
    We argue that it suffices to show $Q(D \cup U, \tau)\subseteq
    Q(D[\tmem] \cup U, \tau)$ for every $\tin$-update $U$ and time point
    $\tau$ satisfying $\tout \leq \tau \leq \tmem + \radius(Q)$.  If
    $\vec{o} \in Q(D \cup U, \tau)$ for $\tau > \tmem + \radius(Q)$,
    then $\vec{o} \in Q(D[\tau - \radius(Q) - 1] \cup U, \tau)$ by
    Lemma~\ref{lem:nr-forget-crit-tp}, and hence $\vec{o} \in Q(D[\tmem] \cup
    U, \tau)$.  Furthermore, by Lemma~\ref{lem:small-updates}, it
    suffices to show that $\vec{o} \in Q(D \cup U, \tau)$ implies
    $\vec{o} \in Q(D[\tmem] \cup U, \tau)$ for every tuple $\vec{o}$,
    $\tin$-update $U$ with time points smaller than or equal to $\tmem + 2\cdot
    \radius(Q)$, and time point $\tau$ satisfying $\tout \leq \tau
    \leq \tmem + \radius(Q)$.
    Now, let $\vec{o} \in Q(D \cup U, \tau)$ 
    for $U$ a $\tin$-update with time points smaller than or equal to $\tmem + 2\cdot\radius(Q)$,
    and $\tau$ satisfying $\tout \leq \tau \leq \tmem + \radius(Q)$.
    We have that $\vec{o} \in Q_1(U, \tau)$ by construction of $Q_1$.
    Hence $\vec{o} \in Q_2(U, \tau)$ by our assumption,
    and hence $\vec{o} \in Q(D[\tmem] \cup U, \tau)$ by construction of $Q_2$,
    as required.

    For the converse, assume $Q_1 \not \sqsubseteq Q_2$.  There is a
    time point $\tau$, dataset $U$ and a tuple $\vec{o}$ such that
    $\vec{o} \in Q_1(U, \tau)$ and $\vec{o} \notin Q_2(U, \tau)$.  In
    particular, $\vec{o} \in Q_1(U', \tau)$ holds for the subset of
    $U$ containing time points $\tau'$ with $\tin < \tau' \leq \tmem +
    2\cdot\radius(Q)$ by construction of $Q_1$.  Note that $U'$ is a
    $\tin$-update.  Furthermore, $\tau$ satisfies $\tout \leq \tau
    \leq \tmem + \radius(Q)$ by construction of $Q_1$.  Then, $\vec{o}
    \in Q_1(U', \tau)$ implies $\vec{o} \in Q(D \cup U', \tau)$, and
    $\vec{o} \notin Q_2(U, \tau)$ implies $\vec{o} \notin Q_2(U',
    \tau)$ by monotonicity of entailment, which implies $\vec{o}
    \notin Q(D[\tmem] \cup U', \tau)$.  Therefore, \textsc{Forget}
    does not hold for $I$.
\end{proof}

\nrforgettingtheorem*
\begin{proof}[Proof sketch]
    Let $Q_1$ and $Q_2$ be the left and right critical queries for $I$.
    In order to check whether $\textsc{Forget}$ holds for $I$, it
    suffices to check $Q_1 \sqsubseteq Q_2$ by
    Lemma~\ref{lem:forgetting-correctness}.  Clearly, $Q_1$ and $Q_2$
    can be built in polynomial time.  We argue next that $Q_1
    \sqsubseteq Q_2$ can be checked in polynomial time.

    Let $E_i^\tau$, for a time point $\tau$, be the (temporal) UCQ
    consisting of (the leaves of) each maximal unfolding of $Q_i$
    starting with the atom $P_Q(\vec{x},\tau)$, and let
    $E_i=\bigvee_{\tout\le\tau\le\tmem+\radius(Q)}E_i^\tau$. Clearly,
    $E_i$ is equivalent to $Q_i$ since, by construction, $Q_i$ can only
    derive facts about $P_Q$ between $\tout$ and
    $\tmem+\radius(Q)$. Note that if $Q$ is fixed, then for each
    $E_i^\tau$, the number of conjuncts in each CQ in $E_i^\tau$ is
    bounded by a constant $c$, the arity of each such conjunct is
    bounded by a constant $a$, and the number of CQs in each $E_i^\tau$
    is bounded by $r^{\rank(\Pi_Q)+1}$, where $r$ is the number of rules
    in $\Pi_i'$; importantly, $r$ and hence $r^{\rank(\Pi_Q)+1}$ is
    polynomial in the size of the input. Moreover, since $Q$ is \connected{}, 
    $E_i$ mentions no time variables, and hence each
    temporal CQ in $E_i$ can be equivalently seen as a nontemporal CQ.

    In order to check the containment $Q_1 \sqsubseteq Q_2$, we can
    equivalently check whether for each CQ $Q_1'$ in $E_1$ there is a CQ
    $Q_2'$ in $E_2$ such that $Q_1' \sqsubseteq Q_2'$; for this, it is
    well-known that we can equivalently check whether there is a
    containment mapping from $Q_2'$ to $Q_1'$.  The number of pairs of
    queries to check is bounded by $(\tmem+\radius(Q)-\tout+1)\cdot
    r^{2(\rank(\Pi_Q)+1)}$.  Furthermore, the size and number of
    possible containment mappings is bounded (resp., polynomially and
    exponentially) in $c\cdot a$. Hence, we can generate all the
    possible containment mappings for a pair of queries in constant
    time, and check them in polynomial time.
\end{proof}

\section{Containment of nonrecursive queries}

\nrtcqinconexptime*
\begin{proof}[Proof sketch]
    We proceed by a reduction
    to datalog query containment.  Without loss of generality,
    $P_{Q_1} = P_{Q_2} = G$ for some predicate $G$.

    If $G$ is rigid, since we have assumed that $Q_1$ and $Q_2$ contain no time points and are \connected{}, $G$ recursively depends only on rigid
    predicates. Therefore, it suffices to consider the containment problem for
    the nonrecursive datalog subprograms of $\Pi_{Q_1}$ and $\Pi_{Q_2}$, which is
    \conexptimecomplete{} by the results in~\cite{benedikt2010impact}.

    If $G$ is temporal, let $m$ be the maximum between the radiuses of
    $\Pi_{Q_1}$ and $\Pi_{Q_2}$, and let $\mathbf{T}$ be the time interval
    $[0,2m]$.  Let $\Pi_i$ be the grounding of $\Pi_{Q_i}$ on the temporal
    arguments with time points in $\mathbf{T}$. Note that, since we have assumed that $Q_1$ and $Q_2$ are \connected{}, the size of $\Pi_i$ is bounded by $m$
    times the size of $\Pi_{Q_i}$, i.e., cubically in the size of $\Pi_{Q_i}$.
    Let $G'$ be a fresh temporal IDB predicate of the same arity as $G$.  Let
    $\Pi_i'$ be $\Pi_i$ extended with the rule
    \begin{equation*}
        G(\vec{x}, m) \to G'(\vec{x}, m)
    \end{equation*}
    Let $Q'_i$ be the query $\langle G', \Pi_i' \rangle$.  We have that
    $Q_1 \sqsubseteq Q_2$ iff $Q_1' \sqsubseteq Q_2'$, because
    \begin{inparaenum}[\it(i)]
    \item derivations of facts at $m$ involve only time points in $\mathbf{T}$ and,
    \item for each dataset $D$ and time point $n$, there is a dataset $D'$ such
        that all derivations of facts at $n$ w.r.t.\ $Q_1\cup D$ and $Q_2\cup D$
        are isomorphic to derivations of facts at $m$ w.r.t.\ $Q_1\cup D'$ and
        $Q_2\cup D'$.
    \end{inparaenum}
    Finally, each $Q_i'$ is temporally ground, and hence can be seen as a datalog
    query.  The claim once again follows by~\cite{benedikt2010impact} since
    $Q_1'$, $Q_2'$ are polynomial in $Q_1$, $Q_2$.
\end{proof}

\section{Proofs for Section~\ref{sec:delay}}

\subsubsection{Delay}

\begin{lemma} \label{lemma:hardnessdidelay} %
    There exists a \logspace{}-computable many-one reduction $\phi$ from
    containment of datalog queries to $\textsc{Delay}$ such that, for every
    instance $I = \langle Q_1, Q_2 \rangle$ of query containment, the query in
    $\phi(I)$ is nonrecursive if $Q_1$ and $Q_2$ are nonrecursive.
\end{lemma}
\begin{proof}
    Let $\langle Q_1, Q_2 \rangle$ be an instance of query containment with $Q_1$
    and $Q_2$ datalog queries.  Without loss of generality,
    $P_{Q_1} = P_{Q_2} = G$.  For any rigid $n$-ary IDB predicate $P$ and for
    $i \in \{ 1,2 \}$, let $P_i$ be a fresh rigid $n$-ary IDB predicate uniquely
    associated with $P$ and $i$.  For any rigid $n$-ary EDB (resp., IDB)
    predicate $P$, let $P^\mathrm{t}$ be a fresh temporal $(n+1)$-ary EDB (IDB)
    predicate uniquely associated with $P$.  Let $t$ be a time variable.  For
    $i \in \{ 1,2 \}$, let $\program_i$ be $\program_{Q_i}$ after replacing each
    rigid $n$-ary IDB predicate $P$ with $P_i$; and let $\program_i'$ be
    $\program_i$ after replacing each rigid atom $P(\vec{u})$ with the temporal
    atom $P^\mathrm{t}(\vec{u}, t)$.  Let $A$ be a fresh temporal unary EDB
    predicate.  Let $Q$ be the query such that $P_Q$ is a fresh temporal IDB
    predicate of the same arity as $G_i^\mathrm{t}$, and $\program_Q$ is
    $\program_1' \cup \program_2'$ extended with the following rules:
    \begin{align}
        \label{eq:didelayruleone} A(t+1) \wedge G_1^\mathrm{t}(\vec{x},t) \rightarrow \alignhere P_Q(\vec{x},t) \\
        \label{eq:didelayruletwo} G_2^\mathrm{t}(\vec{x},t) \rightarrow \alignhere P_Q(\vec{x},t) 
    \end{align}
    It is easily seen that $Q$ can be constructed in logarithmic space w.r.t.\
    the size of $Q_1$ and $Q_2$.

    We show that $Q_1 \sqsubseteq Q_2$ if and only if $\textsc{Delay}$ holds for $\phi(\langle Q_1, Q_2 \rangle) = \langle Q, 0 \rangle$.

    Assume that $Q_1 \sqsubseteq Q_2$ holds.  We prove that
    $\textsc{Delay}(Q, 0)$ is true by showing
    $Q(D \cup U, \tau)\subseteq Q(D, \tau)$ for every dataset $D$, time point
    $\tau$, and $\tau$-update $U$.  Let $\vec{o}$ be a tuple of objects, $D$ a
    dataset, $\tau$ a time point, and let $U$ be a $\tau$-update such that
    $\vec{o} \in Q(D \cup U, \tau)$.  Let $\delta$ be a derivation of
    $P_Q(\vec{o}, \tau)$ from $\Pi_Q \cup D \cup U$.  Then the root of
    $\delta$ is labelled with an instance of either rule~\eqref{eq:didelayruleone} or
    rule~\eqref{eq:didelayruletwo}, since $P_Q$ does not occur in
    $\Pi_1' \cup \Pi_2'$.  First, suppose the root is labelled with an instance of
    rule~\eqref{eq:didelayruleone}.  Clearly,
    $\Pi_Q \cup D \cup U \models G_1^\mathrm{t}(\vec{o}, \tau)$. Since the
    atoms of all rules in $\Pi_Q$ but \eqref{eq:didelayruleone} have $t$ as a
    time argument, no derivation $\delta'$ of $G_1^\mathrm{t}(\vec{o}, \tau)$
    from $\Pi_Q \cup D \cup U$ contains a time point different from $\tau$,
    and hence no atom in $U$; it follows that $\delta'$ is also a derivation
    of $G_1^\mathrm{t}(\vec{o}, \tau)$ from $\Pi_Q \cup D$, and hence
    $\Pi_Q \cup D \models G_1^\mathrm{t}(\vec{o}, \tau)$, which implies
    $\vec{o} \in Q_1(D')$ by construction of $Q$, where $D'$ consists of each
    fact $P(\vec{c})$ for $P(\vec{c}, \tau) \in D$.  Therefore, since
    $Q_1\sqsubseteq Q_2$ by assumption, $\vec{o} \in Q_2(D')$. Thus,
    $\Pi_Q \cup D \models G_2^\mathrm{t}(\vec{o}, \tau)$ by the construction of
    $Q$, and hence $\vec{o} \in Q(D, \tau)$ by rule~\eqref{eq:didelayruletwo}.
    Now, suppose the root of $\delta$ is labelled with an instance of
    rule~\eqref{eq:didelayruletwo}.  Clearly,
    $\Pi_Q \cup D \cup U \models G_2^\mathrm{t}(\vec{o}, \tau)$. Since the
    atoms of all rules in $\Pi_Q$ but \eqref{eq:didelayruleone} have $t$ as a
    time argument, no derivation $\delta'$ of $G_2^\mathrm{t}(\vec{o}, \tau)$
    from $\Pi_Q \cup D \cup U$ contains a time point different from $\tau$,
    and hence no atom in $U$; it follows that $\delta'$ is also a derivation
    of $G_2^\mathrm{t}(\vec{o}, \tau)$ from $\Pi_Q \cup D$, and hence
    $\Pi_Q \cup D \models G_2^\mathrm{t}(\vec{o}, \tau)$.  Thus,
    $\vec{o} \in Q(D, \tau)$ by rule~\eqref{eq:didelayruletwo}.

    For the converse, assume that $\textsc{Delay}(Q, 0)$ is true.  We prove
    $Q_1 \sqsubseteq Q_2$ by showing $Q_1(D)\subseteq Q_2(D)$ for every dataset
    $D$.  Let $\vec{o}$ be a tuple of objects and $D$ a dataset such that
    $\vec{o} \in Q_1(D)$.  By construction,
    $\Pi_Q \cup D' \models G_1^{\mathrm{t}}(\vec{o}, \tau)$ for each time point
    $\tau$, where $D'$ consists of each fact $P(\vec{c}, \tau)$ for
    $P(\vec{c}) \in D$.  By rule~\eqref{eq:didelayruleone},
    $\vec{o} \in Q(D' \cup U, \tau)$, where $U$ is the $\tau$-update containing
    $A(\tau + 1)$.  Since $\textsc{Delay}(Q, 0)$ is true by assumption, we have
    that $\vec{o} \in Q(D', \tau)$.  Since $A(\tau + 1) \notin D'$ and $P_Q$
    does not occur in $\Pi_1' \cup \Pi_2'$, the root of any derivation of
    $P_Q(\vec{o}, \tau)$ from $\Pi_Q \cup D'$ must be labelled with an instance of
    rule~\eqref{eq:didelayruletwo}, and hence
    $\Pi_Q \cup D' \models G_2^\mathrm{t}(\vec{o}, \tau)$.  Therfore,
    $\vec{o} \in Q_2(D)$ by construction.
\end{proof}

\undecidabilitydelay*
\begin{proof}
    The claim follows by Lemma~\ref{lemma:hardnessdidelay}
    since query containment for datalog is undecidable by the results in~\cite{shmueli1993equivalence}.
\end{proof}

\nrdelaytheorem*
\begin{proof}
    We first prove hardness.
    The reduction $\phi$ in Lemma~\ref{lemma:hardnessdidelay} is such that,
    for each instance $I = \langle Q_1, Q_2 \rangle$ of query containment,
    the query in $\phi(I)$ is \connected{},
    and is also nonrecursive constant-free if $Q_1$ and $Q_2$ are nonrecursive constant-free.
    Furthermore,
    query containment is 
    \conexptimehard{} already for nonrecursive constant-free datalog queries,
    by the results in~\cite{benedikt2010impact}.

    For the upper bound, we show that there is a \logspace{}-computable many-one
    reduction $\phi$ from \textsc{Delay} restricted to nonrecursive queries to
    query containment for nonrecursive queries; the result then follows by
    Lemma~\ref{lem:containment}.
    Consider the construction given in Section~\ref{sec:delay}.
    We show that $Q_1 \sqsubseteq Q_2$ iff $\textsc{Delay}(Q, d)$ holds.  

    Assume $Q_1 \sqsubseteq Q_2$.  We show that $\textsc{Delay}(Q, d)$
    holds by showing that $Q(D \cup U, \tin - d)\subseteq Q(D, \tin - d)$ 
    for every $\tin$-history $D$, time point $\tin$, 
    and $\tin$-update $U$. 
    Let $\vec{o} \in Q(D \cup U, \tin - d)$ for any tuple $\vec{o}$,
    time point $\tin$, $\tin$-history $D$, and $\tin$-update $U$.  
    We assume without loss of generality that $A$ does not occur in 
    $D \cup U$ since $A$ does not occur in $\Pi_Q$.  
    We show $\vec{o} \in Q(D, \tin - d)$.  
    Let $D'=\{A(\tin - d)\}$.  
    Since $\Pi_Q \subseteq \Pi_1$, $\vec{o} \in Q(D \cup U, \tin - d)$ 
    implies $\Pi_1 \cup D \cup U \models P_Q(\vec{o}, \tin - d)$; 
    hence, $\vec{o} \in Q_1(D \cup D' \cup U, \tin - d)$ 
    by rule~\eqref{eq:delayruleone}.  
    Then, $\vec{o} \in Q_2(D \cup D' \cup U, \tin - d)$ by our assumption.  
    Note that $\Pi_2 \cup D \cup D' \cup U \not \models B(\tau)$ 
    for any $\tau > \tin$ because $B$ is IDB and can only be derived 
    by rule~\eqref{eq:delayruletwo}, and $A$ does not occur in $D \cup U$; 
    hence no fact $P(\vec{c}, \tau)$ occurs in a derivation of 
    $P_Q(\vec{o}, \tin - d)$ from $\Pi_2 \cup D \cup D' \cup U$, 
    since predicate $P$ occurs only in rule \eqref{eq:delayrulethree}, 
    which requires $B(\tau)$; hence $\vec{o} \in Q_2(D \cup D', \tin - d)$ 
    because $D$ is a $\tin$-history and $U$ is a $\tin$-update;
    and hence $\vec{o} \in Q(D, \tin -d)$ 
    by the construction of $Q_2$.

    Assume $Q_1 \not\sqsubseteq Q_2$, and hence there is a tuple
    $\vec{o}$, a time point $\tau$, and a dataset $D$ such that
    $\vec{o} \in Q_1(D, \tau)$ and $\vec{o} \notin Q_2(D, \tau)$.  Let
    $D'$ contain each fact in $D$ with time argument at most
    $\tin=\tau+d$, and let $U = D \setminus D'$---note that $U$ is a
    $\tin$-update.  We show that $\textsc{Delay}(Q, d)$ does not hold
    by showing $\vec{o} \in Q(D' \cup U, \tin - d)$ and $\vec{o}
    \notin Q(D', \tin - d)$ (where $\tin-d=\tau$).

    We first show $\vec{o} \in Q(D' \cup U, \tin - d)$.  We have that
    $\vec{o} \in Q_1(D' \cup U, \tin - d)$ because $D = D' \cup U$;
    hence $\Pi_1 \cup D' \cup U \models P_Q(\vec{o}, \tin - d)$ by
    rule~\eqref{eq:delayruleone}; hence $\vec{o} \in Q(D' \cup U, \tin
    - d)$ because $\Pi_1$ is $\Pi_Q$ extended with
    rule~\eqref{eq:delayruleone}, which derives $G$ that does not
    occur in $\Pi_Q$.

    We next show $\vec{o} \notin Q(D', \tin - d)$.  Let $D''$ be the
    set consisting of each fact in $D'$ with time argument $\tau$
    satisfying $\tin - d - \radius(Q) \leq \tau$.  Note that each time
    point in $D''$ has time argument $\tau$ satisfying $\tin - d -
    \radius(Q) \leq \tau \leq \tin$, since the time points in $D'$ are
    at most $\tin$.  We have that $\vec{o} \notin Q_2(D, \tin - d)$
    implies $\vec{o} \notin Q_2(D'', \tin - d)$ by monotonicity of
    entailment because $D'' \subseteq D$.  We show by contraposition
    that $\vec{o} \notin Q_2(D'', \tin - d)$ implies $\vec{o} \notin
    Q(D'', \tin - d)$.  Let $\delta$ be a derivation of $P_Q(\vec{o},
    \tin - d)$ from $\Pi_Q \cup D''$.  Let $\delta'$ be the derivation
    obtained from $\delta$ by first adding a fresh root labelled with
    the proper instance of rule~\eqref{eq:delayruleone} and having the
    root of $\delta$ as a child; then replacing each EDB predicate $P$
    from $\Pi_Q$ with the corresponding $P'$; then, for each node $v$
    and for each atom $P'(\vec{c}, \tau)$ in the body of the label of
    $v$, we add a child to $v$ labelled with the proper instance of
    rule~\eqref{eq:delayrulethree}; and finally, for each node having
    $B$ in the body of its label, we add the proper instance of
    rule~\eqref{eq:delayruletwo}.  We have that $\delta'$ is a
    derivation of $G(\vec{o}, \tin -d)$ from $\Pi_2 \cup D''$
    because:
    \begin{inparaenum}[\it(i)]
    \item $A(\tin - d)$ is in $D''$ since $\vec{o} \in Q_1(D' \cup U,
        \tin - d)$, and
    \item there is an instance of rule~\eqref{eq:delayruletwo}
        deriving $B(\tau)$ from $A(\tin - d)$ for each $P(\vec{c}, \tau)
        \in D''$, since we have that $\tau$ satisfies $\tin - d -
        \radius(Q) \leq \tau \leq \tin$ as observed before.
    \end{inparaenum}
    Finally, $\vec{o} \notin Q_2(D'', \tin - d)$ implies $\vec{o}
    \notin Q_2(D', \tin - d)$ by Lemma~\ref{lem:nr-forget-crit-tp}.
\end{proof}

\subsubsection{Window}

\begin{lemma} \label{lemma:hardnesswindow} %
    There exists a \logspace{}-computable many-one reduction $\phi$ from
    containment of datalog queries to $\textsc{Window}$ such that, for every
    instance $I = \langle Q_1, Q_2 \rangle$ of query containment, the query in
    $\phi(I)$ is nonrecursive if $Q_1$ and $Q_2$ are nonrecursive.
\end{lemma}
\begin{proof}
    Let $I = \langle Q_1, Q_2 \rangle$ be an instance of query containment with $Q_1$
    and $Q_2$ datalog queries. 
    Without loss of generality, $P_{Q_1} = P_{Q_2} = G$.  
    For any rigid $n$-ary IDB predicate $P$ and $i \in \{ 1,2 \}$, 
    let $P_i$ be a fresh rigid $n$-ary IDB predicate uniquely
    associated with $P$ and $i$.  
    For any rigid $n$-ary EDB (resp., IDB) predicate $P$, 
    let $P^\mathrm{t}$ be a fresh temporal $(n+1)$-ary EDB (IDB)
    predicate uniquely associated with $P$.  
    Let $t$ be a time variable.  
    For $i \in \{ 1,2 \}$, let $\program_i$ be $\program_{Q_i}$ 
    after replacing each rigid $n$-ary IDB predicate $P$ with $P_i$; 
    let $\program_i'$ be $\program_i$ after replacing each rigid atom $P(\vec{u})$ 
    with the temporal atom $P^\mathrm{t}(\vec{u}, t)$.  
    Let $A$ be a fresh temporal unary EDB predicate.  
    Let $Q$ be the query such that $P_Q$ is a fresh temporal IDB
    predicate of the same arity as $G_1^\mathrm{t}$ (or, equivalently, as
    $G_2^\mathrm{t}$), and $\program_Q$ is $\program_1' \cup \program_2'$
    extended with the following rules:
    \begin{align}
        \label{eq:windowruleone} A(t-1) \wedge G_1^\mathrm{t}(\vec{x},t) \rightarrow \alignhere P_Q(\vec{x},t) \\
        \label{eq:windowruletwo} G_2^\mathrm{t}(\vec{x},t) \rightarrow \alignhere P_Q(\vec{x},t) 
    \end{align}
    Clearly, $Q$ can be constructed in logarithmic space w.r.t.\ the size of
    $Q_1$ and $Q_2$.

    We show that $Q_1 \sqsubseteq Q_2$ iff $\textsc{Window}$ 
    holds for $\phi(I) = \langle Q, 0, 0 \rangle$.

    Assume that $Q_1 \sqsubseteq Q_2$ holds.  
    We show that $\textsc{Window}$
    holds for $\langle Q, 0, 0 \rangle$, by showing that
    $Q(D \cup U, \tout)\subseteq Q(D[\tin] \cup U, \tout)$ 
    for every dataset $D$, 
    time points $\tin$ and $\tout$ 
    such that $\tout > \tin$,
    and $\tin$-update $U$.
    Let $\vec{o}$ be a tuple,
    let $D$ be a dataset,
    let $\tin$ and $\tout$ be time points
    such that $\tout > \tin$,
    and let $U$ be a $\tin$-update
    such that 
    $\vec{o} \in Q(D \cup U, \tout)$.
    Let $\delta$ be a derivation of $P_Q(\vec{o}, \tout)$ 
    from $\Pi_Q \cup D \cup U$.
    The root of $\delta$ is labelled with an instance of either 
    rule~\eqref{eq:windowruleone} or rule~\eqref{eq:windowruletwo} 
    since $P_Q$ does not occur in $\program_1' \cup \program_2'$.
    We consider the two cases separately.  
    If the label is an instance of rule \eqref{eq:windowruleone}, 
    then $\Pi_Q \cup D \cup U \models G_1^\mathrm{t}(\vec{o},\tout)$; 
    then $\Pi_1' \cup D \cup U \models G_1^\mathrm{t}(\vec{o},\tout)$ 
    by the construction of $\Pi_Q$; 
    then $\Pi_1' \cup D[\tin] \cup U \models G_1^\mathrm{t}(\vec{o},\tout)$ because atoms in
    $\Pi_1'$ all have $t$ as time argument and $\tout > \tin$ by assumption;
    then $\vec{o} \in Q_1(D')$ by the construction of $\Pi_1'$,
    where $D'$ is the dataset consisting of each rigid fact $P(\vec{c})$ for
    $P^\mathrm{t}(\vec{c}, \tout) \in D[\tin] \cup U$;
    then $\vec{o} \in Q_2(D')$ since
    $Q_1\sqsubseteq Q_2$ by assumption; then
    $\Pi_2' \cup D[\tin] \cup U \models G_2^\mathrm{t}(\vec{o},\tout)$ by the constructions of
    $\Pi_2'$ and $D'$; 
    then $\Pi_Q \cup D[\tin] \cup U \models G_2^\mathrm{t}(\vec{o},\tout)$ because
    $\Pi_2' \subseteq \Pi_Q$; then $\vec{o} \in Q(D[\tin] \cup U, \tout)$ by
    rule~\eqref{eq:windowruletwo}.
    If the root of $\delta$ is labelled with an
    instance of rule \eqref{eq:windowruletwo}, then
    $\Pi_Q \cup D \cup U \models G_2^\mathrm{t}(\vec{o}, \tout)$; then
    $\Pi_2' \cup D \cup U \models G_2^\mathrm{t}(\vec{o}, \tout)$ by the
    construction of $\Pi_Q$; then
    $\Pi_2' \cup D[\tin] \cup U \models G_2^\mathrm{t}(\vec{o}, \tout)$ because atoms in
    $\Pi_2'$ all have $t$ as time argument and $\tout > \tin$ by assumption; 
    then $\Pi_Q \cup D[\tin] \cup U \models G_2^\mathrm{t}(\vec{o}, \tau)$
    because $\Pi_2' \subseteq \Pi_Q$.

    For the converse,
    assume $Q_1 \not \sqsubseteq Q_2$,
    and hence 
    there is a tuple $\vec{o}$ of objects and a dataset $D$
    such that $\vec{o} \in Q_1(D)$ and $\vec{o} \notin Q_2(D)$.
    We show that \textsc{Window} does not hold on $\langle Q, 0, 0 \rangle$.
    Let $D' = \{ A(0) \}$,
    and let $\tout = 1$ and let $\tin = 0$.
    Let $U$ be the $\tin$-update containing each temporal
    fact $P^\mathrm{t}(\vec{c}, \tout)$ for $P(\vec{c}) \in D$.
    We have that $\vec{o} \in Q_1(D)$ implies 
    $\Pi_Q \cup U \models G_1^\mathrm{t}(\vec{o}, \tout)$
    by the construction of $\Pi_1'$.
    Hence $\vec{o} \in Q(D' \cup U, \tout)$
    by rule~\eqref{eq:windowruleone}.
    We show that that $\vec{o} \notin Q(D'[\tin] \cup U, \tout)$.
    Note that $D'[\tin] = \emptyset$.
    Let us assume by contradiction that $\vec{o} \in Q(U, \tout)$.
    Let $\delta$ be a derivation of $P_Q(\vec{o}, \tout)$ from $\Pi_Q \cup U$.
    The root of $\delta$ is labelled with an instance of either 
    rule~\eqref{eq:windowruleone} or rule~\eqref{eq:windowruletwo} 
    since $P_Q$ does not occur in $\program_1' \cup \program_2'$.
    We discuss the two cases separately.
    If the root of $\delta$ is labelled with an instance of rule~\eqref{eq:windowruleone},
    then $A(\tout - 1) \in U$, which cannot be because $U$ contains no 
    time point different from $\tout$.
    If the root of $\delta$ is labelled with an instance of rule~\eqref{eq:windowruletwo},
    then $\Pi_Q \cup U \models G_2^\mathrm{t}(\vec{o}, \tout)$;
    then $\Pi_2' \cup U \models G_2^\mathrm{t}(\vec{o}, \tout)$ by construction of $\Pi_Q$;
    then $\vec{o} \in Q_2(D)$ by construction of $\Pi_2$ and $U$,
    which contradicts our assumption.
\end{proof}

\undecidabilitywindow*
\begin{proof}
    The claim follows by Lemma~\ref{lemma:hardnesswindow}
    since query containment for datalog is undecidable by the results in~\cite{shmueli1993equivalence}.
\end{proof}

\nrwindowtheorem*
\begin{proof}
    We first prove hardness.
    The reduction $\phi$ in Lemma~\ref{lemma:hardnesswindow} is such that,
    for each instance $I = \langle Q_1, Q_2 \rangle$ of query containment,
    the query in $\phi(I)$ is \connected{},
    and is also nonrecursive constant-free 
    if $Q_1$ and $Q_2$ are nonrecursive constant-free.
    Furthermore,
    query containment is 
    \conexptimehard{} already for nonrecursive constant-free datalog queries,
    by the results in~\cite{benedikt2010impact}.

    For the upper bound, we next show that there is a
    \logspace{}-computable many-one reduction $\phi$ from
    \textsc{Window} restricted to nonrecursive queries to nonrecursive
    query containment; the claim then follows by
    Lemma~\ref{lem:containment}.
    Consider the construction given in Section~\ref{sec:delay}.
    We argue that $Q_1 \sqsubseteq Q_2$ iff $\textsc{Window}(Q, d, s)$ holds.

    For the direction from left to right, suppose $Q_1\sqsubseteq Q_2$,
    and let $D$ be a dataset, $\tin$, $\tout$ time points such that
    $\tout>\tin-d$, and $U$ a $\tin$-update. We need to show $Q(D\cup
    U,\tout)=Q(D[\tin-s]\cup U,\tout)$.  Without loss of generality,
    we show the claim for $\tin=0$ and $U=\emptyset$; since $Q$
    contains no time points, for each $D'$, $U'$ and $\tin'$ we have
    $Q(D,\tout)=Q(D'\cup U',\tout')$, where $\tout=\tout'-\tin'$ and
    $D$ is obtained from $D'\cup U'$ by replacing each temporal fact
    $P(\vec{o},\tau)$ with $P(\vec{o},\tau-\tin')$; thus, $Q(D'\cup
    U',\tout')=Q(D'[\tin'-s]\cup U',\tout')$ holds if and only if so
    does $Q(D,\tout)=Q(D[-s],\tout)$.

    The inclusion $Q(D[-s],\tout)\subseteq Q(D,\tout)$ is immediate by
    monotonicity of entailment. For the other inclusion, suppose
    $\vec{o}\in Q(D,\tout)$. We distinguish two cases.

    If $\tout>-s+\radius(Q)$, the derivation of $P_Q(\vec{o},\tout)$
    from $D\cup\Pi_Q$ does not involve facts at time points before
    $-s+1$; this implies $\Pi_Q\cup D[-s]\models P_Q(\vec{o},\tout)$ and hence
    $\vec{o}\in Q(D[-s],\tout)$.

    Similarly, if $-d<\tout\le-s+\radius(Q)$, we have $\Pi_Q\cup
    D'\models P_Q(\vec{o},\tout)$, where $D'$ is obtained from
    $D[-d-\radius(Q)]$ by additionally removing all temporal facts
    holding after $-s+2\cdot\radius(Q)$, since no derivation of
    $P_Q(\vec{o},\tout)$ from $\Pi_Q\cup D$ involves facts at or
    before $-d-\radius(Q)$, or after $-s+2\cdot\radius(Q)$.  Consider
    $D''=D\cup\{\mathit{In}(0)\}$. By construction, we then have
    $\Pi^{d+\radius(Q)}\cup D''\models P'(\vec{c},\tau)$ if and only
    if $P(\vec{c},\tau)\in D'$. Consequently, we have
    $\Pi^{d+\radius(Q)}\cup D''\models P_Q(\vec{o},\tout)$, and, since
    $-d<\tout\le-s+\radius(Q)$, also $\Pi^{d+\radius(Q)}\cup
    D''\models \mathit{Out}(\tout)$; thus, $\vec{o}\in Q_1(D'',\tout)$. By
    assumption, we obtain $\vec{o}\in Q_2(D'',\tout)$, i.e.,
    $\Pi^{s}\cup D''\models G(\vec{o},\tout)$, and hence $\Pi^{s}\cup
    D''\models P_Q(\vec{o},\tout)$. By construction, any derivation of
    $P_Q(\vec{o},\tout)$ from $\Pi^{s}\cup D''$ can only involve
    temporal facts holding after $-s$. Thus, given a derivation
    $\delta$ of $P_Q(\vec{o},\tout)$ from $\Pi^{s}\cup D''$, by
    replacing each subtree of $\delta$ whose root is labelled by
    $P(\vec{c}',\tau')\land B(\tau')\to P'(\vec{c}',\tau')$ with the
    leaf $P(\vec{c}',\tau')$, we obtain a derivation of
    $P_Q(\vec{o},\tout)$ from $\Pi_Q\cup D[-s]$. Consequently,
    $\vec{o}\in Q(D[-s],\tout)$.

    The direction from right to left is similar.
\end{proof}

\fi

\end{document}